\documentclass{article}

\usepackage[nonatbib,final]{neurips_2020}

\usepackage[utf8]{inputenc} 
\usepackage[T1]{fontenc}    
\usepackage{hyperref}       
\usepackage{url}            
\usepackage{booktabs}       
\usepackage{amsfonts}       
\usepackage{nicefrac}       
\usepackage{microtype}      

\usepackage{amsmath}
\usepackage{amssymb}
\usepackage{pifont}
\newcommand{\cmark}{\ding{51}}%
\newcommand{\xmark}{\ding{55}}%
\usepackage{ntheorem}
\usepackage{color}
\usepackage[noabbrev,capitalize]{cleveref}
\usepackage{overpic}
\usepackage{url}
\usepackage{enumitem}

\newtheorem{theorem}{Theorem}

\newtheorem{corollary}[theorem]{Corollary}

\newtheorem{lemma}[theorem]{Lemma}

\newtheorem{proposition}[theorem]{Proposition}

\newenvironment{proof}[1][Proof]{\textbf{#1.} }{\ \rule{0.5em}{0.5em}}
\DeclareMathOperator{\polylog}{{\rm polylog}}

\DeclareMathOperator*{\esssup}{ess\,sup}
\newcommand{\relu}{\text{ReLU}}
\newcommand{\abs}{\text{abs}}
\newcommand{\sign}{\text{sign}}
\newcommand{\m}{\mathbf{m}}
\newcommand{\n}{\mathbf{n}}
\newcommand{\x}{\mathbf{x}}

\title{Rational neural networks}

\author{
  Nicolas Boull\'e\\
  Mathematical Institute\\
  University of Oxford\\
  Oxford, OX2 6GG, UK \\
  \texttt{boulle@maths.ox.ac.uk}\\  
  \And
  Yuji Nakatsukasa\\
  Mathematical Institute\\
  University of Oxford\\
  Oxford, OX2 6GG, UK \\
  \texttt{nakatsukasa@maths.ox.ac.uk}\\
  \And
  Alex Townsend\\
  Department of Mathematics\\
  Cornell University\\
  Ithaca, NY 14853, USA\\
  \texttt{townsend@cornell.edu}
}

\begin{document}

\maketitle

\begin{abstract}
We consider neural networks with rational activation functions. The choice of the nonlinear activation function in deep learning architectures is crucial and heavily impacts the performance of a neural network. We establish optimal bounds in terms of network complexity and prove that rational neural networks approximate smooth functions more efficiently than ReLU networks with exponentially smaller depth. The flexibility and smoothness of rational activation functions make them an attractive alternative to ReLU, as we demonstrate with numerical experiments.
\end{abstract}

\section{Introduction}

Deep learning has become an important topic across many domains of science due to its recent success in image recognition, speech recognition, and drug discovery~\cite{hinton2012deep,krizhevsky2012imagenet,lecun2015deep,ma2015deep}. Deep learning techniques are based on neural networks, which contain a certain number of layers to perform several mathematical transformations on the input. A nonlinear transformation of the input determines the output of each layer in the neural network: $x\mapsto \sigma(Wx+b)$, where $W$ is a matrix called the weight matrix, $b$ is a bias vector, and $\sigma$ is a nonlinear function called the activation function (also called activation unit). The computational cost of training a neural network depends on the total number of nodes (size) and the number of layers (depth). A key question in designing deep learning architectures is the choice of the activation function to reduce the number of trainable parameters of the network while keeping the same approximation power~\cite{goodfellow2016deep}. 

While smooth activation functions such as sigmoid, logistic, or hyperbolic tangent are widely used, they suffer from the ``vanishing gradient problem''~\cite{bengio1994learning} because their derivatives are zero for large inputs. Neural networks based on polynomial activation functions are an alternative~\cite{cheng2018polynomial,daws2019polynomial, goyal2019learning, guarnieri1999multilayer,ma2005constructive,vecci1998learning}, but can be numerically unstable due to large gradients for large inputs~\cite{bengio1994learning}. Moreover, polynomials do not approximate non-smooth functions efficiently~\cite{trefethen2019approximation}, which can lead to optimization issues in classification problems. A popular choice of activation function is the Rectified Linear Unit (ReLU) defined as $\relu(x)=\max(x,0)$~\cite{jarrett2009best,nair2010rectified}. It has numerous advantages, such as being fast to evaluate and zero for many inputs~\cite{glorot2011deep}. Many theoretical studies characterize and understand the expressivity of shallow and deep ReLU neural networks from the perspective of approximation theory~\cite{devore1989optimal,liang2016deep,mhaskar1996neural,telgarsky2016benefits,yarotsky2017error}.

ReLU networks also suffer from drawbacks, which are most evident during training. The main disadvantage is that the gradient of ReLU is zero for negative real numbers. Therefore, its derivative is zero if the activation function is saturated~\cite{maas2013rectifier}. To tackle these issues, several adaptations to ReLU have been proposed such as Leaky ReLU~\cite{maas2013rectifier}, Exponential Linear Unit (ELU)~\cite{clevert2015fast}, Parametric Linear Unit (PReLU)~\cite{he2015delving}, and Scaled Exponential Linear Unit (SELU)~\cite{klambauer2017self}. These modifications outperform ReLU in image classification applications, and some of these activation functions have trainable parameters, which are learned by gradient descent at the same time as the weights and biases of the network.  To obtain significant benefits for image classification and partial differential equation (PDE) solvers, one can perform an exhaustive search over trainable activation functions constructed from standard units~\cite{jagtap2020adaptive,ramachandran2017searching}. However, most of the ``exotic'' activation functions in the literature are motivated by empirical results and are not supported by theoretical statements on their potentially improved approximation power over ReLU.

In this work, we study rational neural networks, which are neural networks with activation functions that are trainable rational functions. In~\cref{sec_th_result}, we provide theoretical statements quantifying the advantages of rational neural networks over ReLU networks. In particular, we remark that a composition of low-degree rational functions has a good approximation power but a relatively small number of trainable parameters. Therefore, we show that rational neural networks require fewer nodes and exponentially smaller depth than ReLU networks to approximate smooth functions to within a certain accuracy. This improved approximation power has practical consequences for large neural networks, given that a deep neural network is computationally expensive to train due to expensive gradient evaluations and slower convergence. The experiments conducted in~\cref{sec_experiments} demonstrate the potential applications of these rational networks for solving PDEs and Generative Adversarial Networks (GANs).\footnote{All code and hyper-parameters are publicly available at~\cite{boulleGit}.}
The practical implementation of rational networks is straightforward in the TensorFlow framework and consists of replacing the activation functions by trainable rational functions. Finally, we highlight the main benefits of rational networks: the fast approximation of functions, the trainability of the activation parameters, and the smoothness of the activation function.

\section{Rational neural networks}

We consider neural networks whose activation functions consist of rational functions with trainable coefficients $a_i$ and $b_j$, i.e., functions of the form:
\begin{equation} \label{eq_rational}
F(x) = \frac{P(x)}{Q(x)}=\frac{\sum_{i=0}^{r_P} a_ix^i}{\sum_{j=0}^{r_Q} b_jx^j}, \qquad a_P\neq0, \quad b_Q\neq 0,
\end{equation}
where $r_P$ and $r_Q$ are the polynomial degrees of the numerator and denominator, respectively.  We say that $F(x)$ is of type $(r_P,r_Q)$ and degree $\max(r_P,r_Q)$.

The use of rational functions in deep learning is motivated by the theoretical work of Telgarsky, who proved error bounds on the approximation of ReLU neural networks by high-degree rational functions and vice versa~\cite{telgarsky2017neural}. On the practical side, neural networks based on rational activation functions are considered by Molina et al.~\cite{molina2019pad}, who defined a safe Pad\'e Activation Unit (PAU) as
\[
F(x) = \frac{\sum_{i=0}^{r_P} a_ix^i}{1+|\sum_{j=1}^{r_Q} b_jx^j|}.
\]
The denominator is selected so that $F(x)$ does not have poles located on the real axis. PAU networks can learn new activation functions and are competitive with state-of-the-art neural networks for image classification. However, this choice results in a non-smooth activation function and makes the gradient expensive to evaluate during training. In a closely related work, Chen et al.~\cite{chen2018rational} propose high-degree rational activation functions in a neural network, which have benefits in terms of approximation power.  However, this choice can significantly increase the number of parameters in the network, causing the training stage to be computationally expensive. 

In this paper, we use low-degree rational functions as activation functions, which are then composed together by the neural network to build high-degree rational functions. In this way, we can leverage the approximation power of high-degree rational functions without making training expensive. We highlight the approximation power of rational networks and provide optimal error bounds to demonstrate that rational neural networks theoretically outperform ReLU networks. Motivated by our theoretical results, we consider rational activation functions of type $(3,2)$, i.e., $r_P=3$ and $r_Q=2$. This type appears naturally in the theoretical analysis due to the composition property of Zolotarev sign functions (see \cref{sec_approx_relu_rat}): the degree of the overall rational function represented by the rational neural network is a whopping $3^{\#\textup{layers}}$, while the number of trainable parameters only grows linearly with respect to the depth of the network. Moreover, a superdiagonal type $(3,2)$ allows the rational activation function to behave like a nonconstant linear function at $\pm\infty$, unlike a diagonal type, e.g.,~$(2,2)$, or the ReLU function. A low-degree activation function keeps the number of trainable parameters small, while the implicit composition in a neural network gives us the approximation power of high-degree rationals. This choice is also motivated empirically, and we do not claim that the type $(3,2)$ is the best choice for all situations as the configurations may depend on the application (see \cref{fig_rational_loss_2d} of the Supplementary Material).
Our experiments on the approximation of smooth functions and GANs suggest that rational neural networks are an attractive alternative to ReLU networks (see \cref{sec_experiments}). We observe that a good initialization, motivated by the theory of rational functions, prevents rational neural networks from having arbitrarily large values.

\section{Theoretical results on rational neural networks} \label{sec_th_result}
Here, we demonstrate the theoretical benefit of using neural networks based on rational activation functions due to their superiority over ReLU in approximating functions. We derive optimal bounds in terms of the total number of trainable parameters (also called size) needed by rational networks to approximate ReLU networks as well as functions in the Sobolev space $\mathcal{W}^{n,\infty}([0,1]^d)$. Throughout this paper, we take $\epsilon$ to be a small parameter with $0<\epsilon<1$. We show that an $\epsilon$-approximation on the domain $[-1,1]^d$ of a ReLU network by a rational neural network must have the following size (indicated in brackets):
\begin{equation} \label{eq_rat_approx_relu}
\text{Rational } [\Omega(\log(\log(1/\epsilon)))] \leq \relu \leq \text{Rational } [\mathcal{O}(\log(\log(1/\epsilon)))],
\end{equation}
where the constants only depend on the size and depth of the ReLU network. Here, the upper bound means that all ReLU networks can be approximated to within $\epsilon$ by a rational network of size $\mathcal{O}(\log(\log(1/\epsilon)))$. The lower bound means that there is a ReLU network that cannot be $\epsilon$-approximated by a rational network of size less than $C\log(\log(1/\epsilon))$, for some constant $C>0$. In comparison, the size needed by a ReLU network to approximate a rational neural network within the tolerance of $\epsilon$ is given by the following inequalities:
\begin{equation} \label{eq_relu_approx_rat}
\text{ReLU } [\Omega(\log(1/\epsilon))] \leq \text{Rational} \leq \text{ReLU } [\mathcal{O}(\log(1/\epsilon))^3],
\end{equation}
where the constants only depend on the size and depth of the rational neural network. This means that all rational networks can be approximated to within $\epsilon$ by a ReLU network of size $\mathcal{O}(\log(1/\epsilon))^3$, while there is a rational network that cannot be $\epsilon$-approximated by a ReLU network of size less than $\Omega(\log(1/\epsilon))$. A comparison between~\eqref{eq_rat_approx_relu} and~\eqref{eq_relu_approx_rat} suggests that rational networks could be more resourceful than ReLU. A key difference between rational networks and neural networks with polynomial activation functions is that polynomials perform poorly on non-smooth functions such as ReLU, with an algebraic convergence of $\mathcal{O}(1/\textup{degree})$~\cite{trefethen2019approximation} rather than the (root-)exponential convergence with rationals (see \cref{fig_init_rat}~(left)).

\subsection{Approximation of ReLU networks by rational neural networks} \label{sec_approx_relu_rat}

Telgarsky showed that neural networks and rational functions can approximate each other in the sense that there exists a rational function of degree\footnote{A polylogarithmic function in $x$ is any polynomial in $\log(x)$ and is denoted by $\polylog(x)$.} $\mathcal{O}(\polylog(1/\epsilon))$ that is $\epsilon$-close to a ReLU network~\cite{telgarsky2017neural}, where $\epsilon>0$ is a small number. To prove this statement, Telgarsky used a rational function constructed with Newman polynomials~\cite{newman1964rational} to obtain a rational approximation to the ReLU function that converges with square-root exponential accuracy. That is, Telgarsky needed a rational function of degree $\Omega(\log(1/\epsilon)^2)$ to achieve a tolerance of $\epsilon$. A degree $r$ rational function can be represented with $2(r+1)$ coefficients, i.e., $a_0,\ldots,a_r$ and $b_0,\ldots, b_r$ in~\cref{eq_rational}. Therefore, the rational approximation to a ReLU network constructed by Telgarsky requires at least $\Omega(\polylog(1/\epsilon))$ parameters. In contrast, for any rational function, Telgarsky showed that there exists a ReLU network of size $\mathcal{O}(\polylog(1/\epsilon))$ that is an $\epsilon$-approximation on $[0,1]^d$.

Our key observation is that by composing low-degree rational functions together, we can approximate a ReLU network much more efficiently in terms of the size (rather than the degree) of the rational network. Our theoretical work is based on a family of rationals called Zolotarev sign functions, which are the best rational approximation on $[-1,-\ell]\cup[\ell,1]$, with $0<\ell<1$, to the $\sign$ function~\cite{achieser2013theory,petrushev2011rational}, defined as
\[\sign(x) = 
\begin{cases}
-1, & x<0,\\
0, & x=0,\\
1, & x>0.
\end{cases}\]
A composition of $k\geq 1$ Zolotarev sign functions of type $(3,2)$ has type $(3^k,3^k-1)$ but can be represented with $7k$ parameters instead of $2\times 3^k+1$. This property enables the construction of a rational approximation to ReLU using compositions of low-degree Zolotarev sign functions with $\mathcal{O}(\log(\log(1/\epsilon)))$ parameters in \cref{lem_relu_rat}. 

\begin{lemma} \label{lem_relu_rat}
Let $0<\epsilon <1$. There exists a rational network $R:[-1,1]\to[-1,1]$ of size $\mathcal{O}(\log(\log(1/\epsilon)))$ such that 
\[\|R-\relu\|_{\infty}:=\max_{x\in[-1,1]}|R(x)-\relu(x)| \leq \epsilon.\] 
Moreover, no rational network of size smaller than $\Omega(\log(\log(1/\epsilon)))$ can achieve this.
\end{lemma}

The proof of \cref{lem_relu_rat} (see Supplementary Material) shows that the given bound is optimal in the sense that a rational network requires at least $\Omega(\log(\log(1/\epsilon)))$ parameters to approximate the ReLU function on $[-1,1]$ to within the tolerance $\epsilon>0$. The convergence of the Zolotarev sign functions to the ReLU function is much faster, with respect to the number of parameters, than the rational constructed with Newman polynomials (see~\cref{fig_init_rat} (left)). 

\begin{figure}[htbp]
\centering
\begin{overpic}[width=0.9\textwidth]{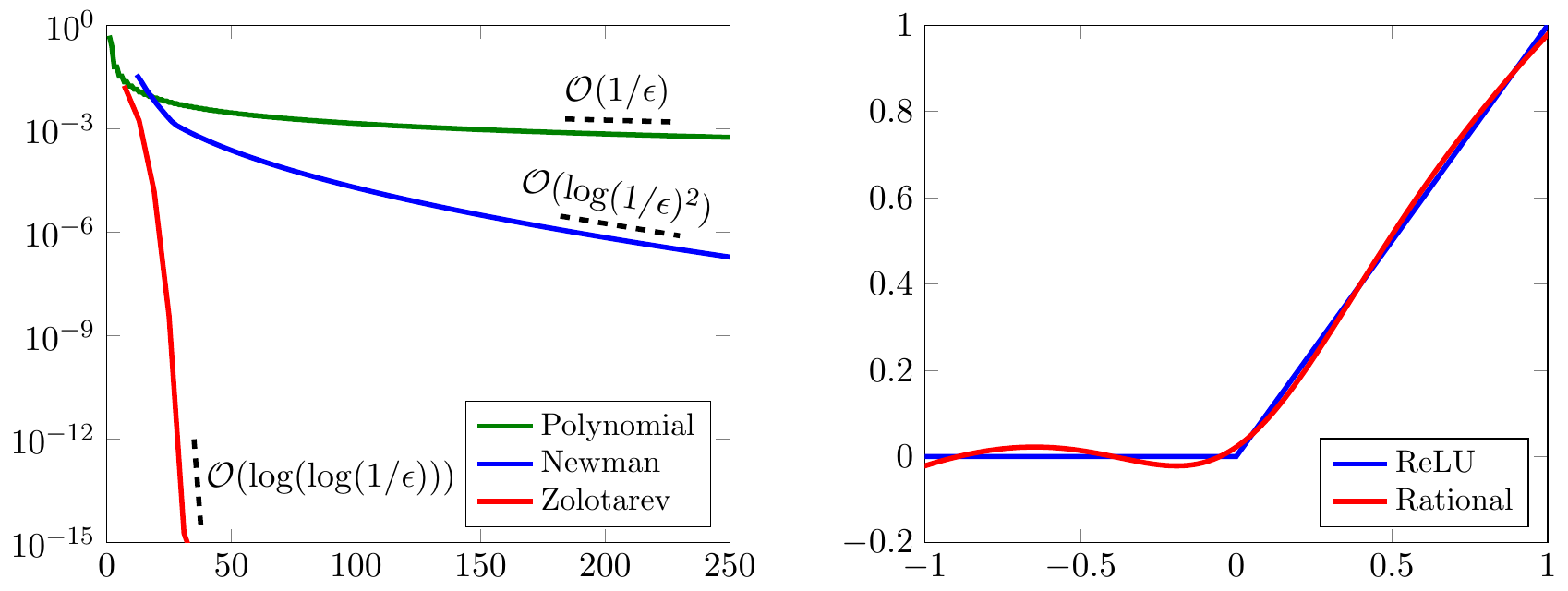}
\put(14,-2.5){Number of parameters}
\put(-3,11){\rotatebox{90}{$\|\relu-r_N\|_{\infty}$}}
\put(78.3,-2.5){x}
\end{overpic}
\vspace{0.3cm}
\caption{Left: Approximation error $\|\relu-r_N\|_{\infty}$ of the Newman (blue), Zolotarev sign functions (red), and best polynomial approximation~\cite{pachon2009barycentric} of degree $N-1$ (green) $r_N$ to ReLU with respect to the number of parameters required to represent $r_N$. Right: Best rational function of type $(3,2)$ (red) that approximates the ReLU function (blue). We use this to initialize the rational activation functions when training a rational neural network.}
\label{fig_init_rat}
\end{figure}

The converse of \cref{lem_relu_rat}, which is a consequence of a theorem proved by Telgarsky~\cite[Theorem~1.1]{telgarsky2017neural}, shows that any rational function can be approximated by a ReLU network of size at most $\mathcal{O}(\log(1/\epsilon)^3)$.

\begin{lemma} \label{lem_relu_approx}
Let $0<\epsilon<1$. If $R:[-1,1]\to[-1,1]$ is a rational function, then there exists a ReLU network $f:[-1,1]\to[-1,1]$ of size $\mathcal{O}(\log(1/\epsilon)^3)$ such that $\|R-f\|_{\infty}\leq \epsilon$.
\end{lemma}

To demonstrate the improved approximation power of rational neural networks over ReLU networks ($\mathcal{O}(\log(\log(1/\epsilon)))$ versus $\mathcal{O}(\log(1/\epsilon)^3)$), it is known that a ReLU networks that approximates $x^2$, which is rational, to within $\epsilon$ on $[-1,1]$ must be of size at least $\Omega(\log(1/\epsilon))$~\cite[Theorem~11]{liang2016deep}. 

We can now state our main theorem based on~\cref{lem_relu_rat,lem_relu_approx}. \cref{th_rat_network} provides bounds on the approximation power of ReLU networks by rational neural networks and vice versa. We regard~\cref{th_rat_network} as an analogue of~\cite[Theorem~1.1]{telgarsky2017neural} for our Zolotarev sign functions, where we are counting the number of training parameters instead of the degree of the rational functions. In particular, our rational networks have high degrees but can be represented with few parameters due to compositions, making training more computationally efficient. While Telgarsky required a rational function with $\mathcal{O}(k^M\log(M/\epsilon)^M)$ parameters to approximate a ReLU network with fewer than $k$ nodes in each of $M$ layers to within a tolerance of $\epsilon$, we construct a rational network that only has size $\mathcal{O}(kM\log(\log(M/\epsilon)))$.

\begin{theorem} \label{th_rat_network}
Let $0<\epsilon <1$ and let $\|\cdot\|_1$ denote the vector 1-norm. The following two statements hold: 
\begin{enumerate}[leftmargin=0cm,itemindent=.5cm,noitemsep]
\item Let $R:[-1,1]^d \to [-1,1]$ be a rational network with $M$ layers and at most $k$ nodes per layer, where each node computes $x\mapsto r(a^\top x+b)$ and $r$ is a rational function with Lipschitz constant $L$ ($a$, $b$, and $r$ are possibly distinct across nodes). Suppose further that $\|a\|_1+|b|\leq 1$ and $r:[-1,1]\rightarrow[-1,1]$. Then, there exists a ReLU network $f:[-1,1]^d\to[-1,1]$ of size
\[
\mathcal{O}\left(kM\log(ML^M/\epsilon)^3\right)
\]
such that $\max_{x\in[-1,1]^d} |R(x)-f(x)| \leq \epsilon$.  
\item  Let $f:[-1,1]^d \to [-1,1]$ be a ReLU network with $M$ layers and at most $k$ nodes per layer, where each node computes $x\mapsto\relu(a^\top x+b)$ and the pair $(a,b)$ (possibly distinct across nodes) satisfies $\|a\|_1+|b| \leq 1$. Then, there exists a rational network $R:[-1,1]^d \to [-1,1]$ of size
\[
\mathcal{O}(kM\log(\log(M/\epsilon)))
\]
such that $\max_{x\in[-1,1]^d} |f(x)-R(x)| \leq \epsilon$.\end{enumerate}
\end{theorem}

\Cref{th_rat_network} highlights the improved approximation power of rational neural networks over ReLU networks. ReLU networks of size $\mathcal{O}(\polylog(1/\epsilon))$ are required to approximate rational networks while rational networks of size only $\mathcal{O}(\log(\log(1/\epsilon)))$ are sufficient to approximate ReLU networks.

\subsection{Approximation of functions by rational networks} \label{sec_func_rat}

A popular question is the required size and depth of deep neural networks to approximate smooth functions~\cite{liang2016deep,montanelli2019deep,yarotsky2017error}. In this section, we consider the approximation theory of rational networks. In particular, we consider the approximation of functions in the Sobolev space $\mathcal{W}^{n,\infty}([0,1]^d)$, where $n\geq 1$ is the regularity of the functions and $d\geq 1$. The norm of a function $f\in\mathcal{W}^{n,\infty}([0,1]^d)$ is defined as
\[
\|f\|_{\mathcal{W}^{n,\infty}([0,1]^d)}=\max_{|\mathbf{n}|\leq n} \esssup_{\mathbf{x}\in[0,1]^d}|D^{\mathbf{n}}f(\mathbf{x})|,
\]
where $\mathbf{n}$ is the multi-index $\mathbf{n}=(n_1,\ldots,n_d)\in\{0,\ldots,n\}^d$, and $D^{\mathbf{n}}f$ is the corresponding weak derivative of $f$. In this section, we consider the approximation of functions from
\[
F_{d,n} := \{f\in \mathcal{W}^{n,\infty}([0,1]^d),\quad \|f\|_{\mathcal{W}^{n,\infty}([0,1]^d)}\leq 1\}.
\]
By the Sobolev embedding theorem~\cite{brezis2010functional}, this space contains the functions in $\mathcal{C}^{n-1}([0,1]^d)$, which is the class of functions whose first $n-1$ derivatives are Lipschitz continuous. 
Yarotsky derived upper bounds on the size of neural networks with piecewise linear activation functions needed to approximate functions in $F_{d,n}$~\cite[Theorem~1]{yarotsky2017error}. In particular, he constructed an $\epsilon$-approximation to functions in $F_{d,n}$ with a ReLU network of size at most $\mathcal{O}(\epsilon^{-d/n}\log(1/\epsilon))$ and depth smaller than $\mathcal{O}(\log(1/\epsilon))$. The term $\epsilon^{-d/n}$ is introduced by a local Taylor approximation, while the $\log(1/\epsilon)$ term is the size of the ReLU network needed to approximate monomials, i.e., $x^j$ for $j\geq 0$, in the Taylor series expansion.

We now present an analogue of Yarotsky's theorem for a rational neural network.

\begin{theorem} \label{th_approx_smooth_rat}
Let $d\geq 1$, $n\geq 1$, $0<\epsilon<1$, and $f\in F_{d,n}$. There exists a rational neural network $R$ of size
\[\mathcal{O}(\epsilon^{-d/n}\log(\log(1/\epsilon)))\]
and maximum depth $\mathcal{O}(\log(\log(1/\epsilon)))$ such that $\|f-R\|_{\infty}\leq \epsilon$.
\end{theorem}

The proof of \cref{th_approx_smooth_rat} consists of approximating $f$ by a local Taylor expansion. One needs to approximate the piecewise linear functions and monomials arising in the Taylor expansion by rational networks using \cref{lem_relu_rat} and \cref{prop_rat_x_n} (see Supplementary Material). The main distinction between Yarotsky's argument and the proof of~\cref{th_approx_smooth_rat} is that monomials can be represented by rational neural networks with a size that does not depend on the accuracy of $\epsilon$. In contrast, ReLU networks require $\mathcal{O}(\log(1/\epsilon))$ parameters. Meanwhile, while ReLU neural networks can exactly approximate piecewise linear functions with a constant number of parameters, rational networks can approximate them with a size of a most $\mathcal{O}(\log(\log(1/\epsilon)))$ (see~\cref{lem_relu_rat}). That is, rational neural networks approximate piecewise linear functions much faster than ReLU networks approximate polynomials. This allows the existence of a rational network approximation to $f$ with exponentially smaller depth ($\mathcal{O}(\log(\log(1/\epsilon)))$) than the ReLU networks constructed by Yarotsky.

A theorem proved by DeVore et al.~\cite{devore1989optimal} gives a lower bound of $\Omega(\epsilon^{-d/n})$ on the number of parameters needed by a neural network to express any function in $F_{d,n}$ with an error $\epsilon$, under the assumption that the weights are chosen continuously. Comparing $\mathcal{O}(\epsilon^{-d/n}\log(\log(1/\epsilon)))$ and $\mathcal{O}(\epsilon^{-d/n}\log(1/\epsilon))$, we find that rational neural networks require exponentially fewer nodes than ReLU networks with respect to the optimal bound of $\Omega(\epsilon^{-d/n})$ to approximate functions in $F_{d,n}$.

\section{Experiments using rational neural networks} \label{sec_experiments}

In this section, we consider neural networks with trainable rational activation functions of type $(3,2)$. We select the type $(3,2)$ based on empirical performance; roughly, a low-degree (but higher than $1$) rational function is ideal for generating high-degree rational functions by composition, with a small number of parameters. The rational activation units can be easily implemented in the open-source TensorFlow library~\cite{abadi2016tensorflow} by using the \texttt{polyval} and \texttt{divide} commands for function evaluations. The coefficients of the numerators and denominators of the rational activation functions are trainable parameters, determined at the same time as the weights and biases of the neural network by backpropagation and a gradient descent optimization algorithm.

One crucial question is the initialization of the coefficients of the rational activation functions~\cite{chen2018rational,molina2019pad}. A badly initialized rational function might contain poles on the real axis, leading to exploding values, or converge to a local minimum in the optimization process. Our experiments, supported by the empirical results of Molina et al.~\cite{molina2019pad}, show that initializing each rational function with the best rational approximation to the ReLU function (as described in \cref{lem_relu_rat}) produces good performance. The underlying idea is to initialize rational networks near a network with ReLU activation functions, widely used for deep learning. Then, the adaptivity of the rational functions allows for further improvements during the training phase. We represent the initial rational function used in our experiments in \cref{fig_init_rat} (right). The coefficients of this function are obtained by using the \texttt{minimax} command, available in the Chebfun software~\cite{driscoll2014chebfun,filip2018rational} for numerically computing rational approximations (see \cref{table_coeffs} in the Supplementary Material).

In the following experiments, we use a single rational activation function of type $(3,2)$ at each layer, instead of different functions at each node to reduce the number of trainable parameters and the computational training expense.

\subsection{Approximation of functions} \label{sec_approx_fun}

Raissi, Perdikaris, and Karniadakis~\cite{raissi2018deep,raissi2019physics} introduce a framework called \emph{deep hidden physics models} for discovering nonlinear partial differential equations (PDEs) from observations. This technique requires to solving the following interpolation problem: given the observation data $(u_i)_{1\leq i\leq N}$ at the spatio-temporal points $(x_i,t_i)_{1\leq i\leq N}$, find a neural network $\mathcal{N}$ (called the identification network), that minimizes the loss function
\begin{equation} \label{eq_loss_approx}
\mathcal{L}=\frac{1}{N}\sum_{i=1}^N|\mathcal{N}(x_i,t_i)-u_i|^2.
\end{equation}
This technique has successfully discovered hidden models in fluid mechanics~\cite{raissi2020hidden}, solid mechanics~\cite{haghighat2020deep}, and nonlinear partial differential equations such as the Korteweg--de Vries (KdV) equation~\cite{raissi2019physics}. Raissi et al. use an identification network, consisting of $4$ layers and $50$ nodes per layer, to interpolate samples from a solution to the KdV equation. Moreover, they observe that networks based on smooth activation functions, such as the hyperbolic tangent ($\tanh(x)$) or the sinusoid ($\sin(x)$), outperform ReLU neural networks~\cite{raissi2018deep,raissi2019physics}. However, the performance of these smooth activation functions highly depends on the application.

Moreover, these functions might not be adapted to approximate non-smooth or highly oscillatory solutions. Recently, Jagtap, Kawaguchi, and Karnidakis~\cite{jagtap2020adaptive} proposed and analyzed different adaptive activation functions to approximate smooth and discontinuous functions with physics-informed neural networks. More specifically, they use an adaptive version of classical activation functions such as sigmoid, hyperbolic tangent, ReLU, and Leaky ReLU. The choice of these trainable activation functions introduces another parameter in the design of the neural network architecture, which may not be ideal for use for a black-box data-driven PDE solver. 

\begin{figure}[htbp]
\centering
\begin{overpic}[width=0.45\textwidth]{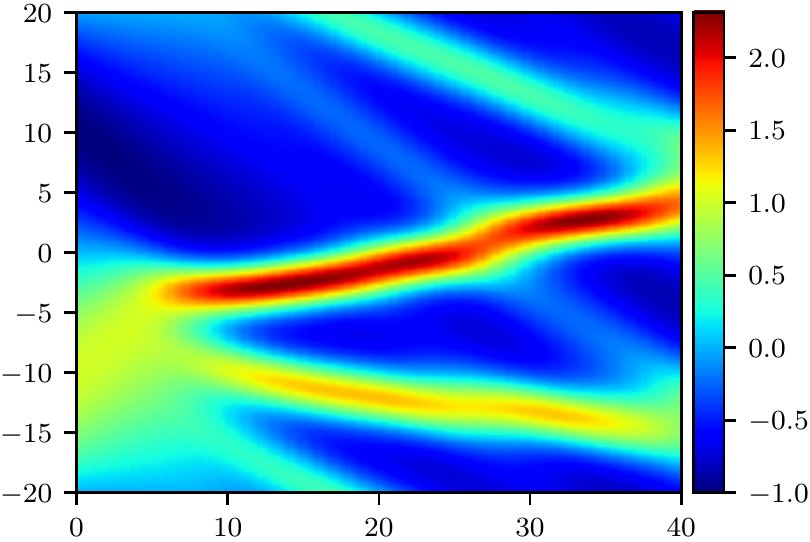}
\put(46,-4){t}
\put(-1,34.3){x}
\end{overpic}
\hspace{0.7cm}
\begin{overpic}[width=0.44\textwidth, trim=30 15 45 25,clip]{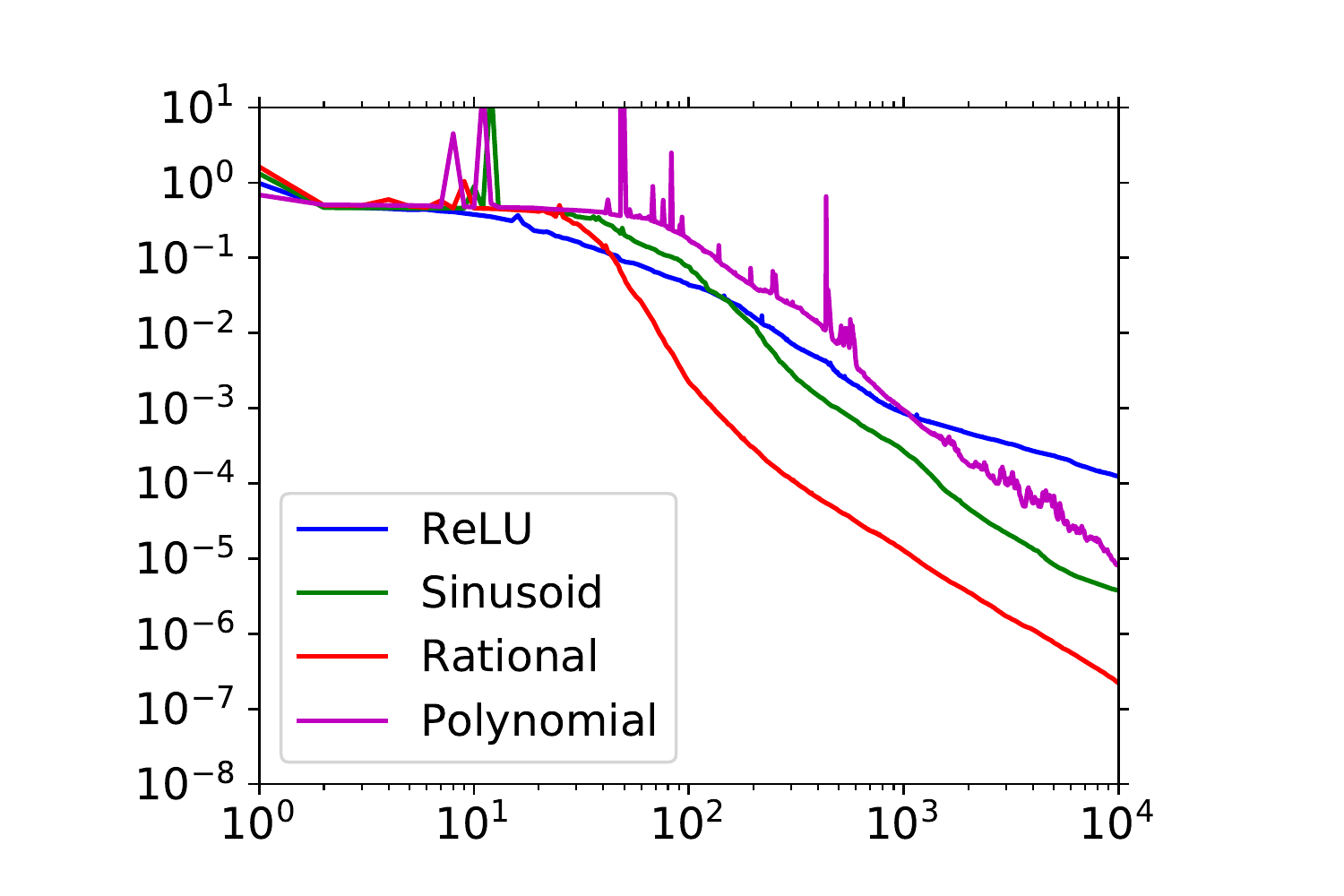}
\put(45,-5){Epochs}
\put(-7,17){\rotatebox{90}{Validation loss}}
\end{overpic}
\vspace{0.2cm}
\caption{Solution to the KdV equation used as training data (left) and validation loss of a ReLU (blue), sinusoid (green), rational (red), and polynomial (purple) neural networks with respect to the number of optimization steps (right).}
\label{fig_loss_2d}
\end{figure}

We illustrate that rational neural networks can address the issues mentioned above due to their adaptivity and approximation power (see \cref{sec_th_result}). Similarly to Raissi~\cite{raissi2018deep}, we use a solution $u$ to the KdV equation:
\[
u_t=-uu_x-u_{xxx}, \quad u(x,0) = -\sin(\pi x/20),
\]
as training data for the identification network (see the left panel of \cref{fig_loss_2d}). We train and compare four neural networks, which contain ReLU, sinusoid, rational, and polynomial activation functions, respectively.\footnote{Details of the parameters used for this experiment are available in the Supplementary Material.} The mean squared error (MSE) of the neural networks on the validation set throughout the training phase is reported in the right panel of \cref{fig_loss_2d}. We observe that the rational neural network outperforms the sinusoid network, despite having the same asymptotic convergence rate. The network with polynomial activation functions (chosen to be of degree 3 in this example) is harder to train than the rational network, as shown by the non-smooth validation loss (see the right panel of \cref{fig_loss_2d}). We highlight that rational neural networks are never much bigger in terms of trainable parameters than ReLU networks since the increase is only linear with respect to the number of layers. Here, the ReLU network has $8000$ parameters (consisting of weights and biases), while the rational network has $8000+7\times\#\textup{layers}=8035$. The  ReLU, sinusoid, rational, and polynomial networks achieve the following mean square errors after $10^4$ epochs:
\begin{align*}
&\text{MSE}(u_{\text{ReLU}}) = 1.9\times 10^{-4}, &\text{MSE}(u_{\sin}) = 3.3\times 10^{-6},\\ 
&\text{MSE}(u_{\text{rat}}) = 1.2\times 10^{-7},  & \text{MSE}(u_{\text{poly}}) = 3.6\times 10^{-5}.
\end{align*}
The absolute approximation errors between the different neural networks and the exact solution to the KdV equation is illustrated in \cref{fig_2d_approx} of the Supplementary Material. The rational neural network is approximatively five times more accurate than the sinusoid network used by Raissi and twenty times more accurate than the ReLU network. Moreover, the approximation errors made by the ReLU network are not uniformly distributed in space and time and located in specific regions, indicating that a network with non-smooth activation functions is not appropriate to resolve smooth solutions to PDEs.

\subsection{Generative adversarial networks}

Generative adversarial networks (GANs) are used to generate fake examples from an existing dataset~\cite{goodfellow2014generative}. They usually consist of two networks: a generator to produce fake samples and a discriminator to evaluate the samples of the generator with the training dataset. Radford et al.~\cite{radford2015unsupervised} describe deep convolutional generative adversarial networks (DCGANs) to build good image representations using convolutional architectures. They evaluate their model on the MNIST and Imagenet image datasets~\cite{deng2009imagenet,lecun1998gradient}. This section highlights the simplicity of using rational activation functions in existing neural network architectures by training an Auxiliary Classifier GAN (ACGAN)~\cite{odena2017conditional} on the MNIST dataset. In particular, the neural network\footnote{We use the TensorFlow implementation available at~\cite{KerasDoc} and provide extended details and results of the experiment in the Supplementary Material.}, denoted by ReLU network in this section, consists of convolutional generator and discriminator networks with ReLU and Leaky ReLU~\cite{maas2013rectifier} activation units (respectively) and is used as a reference GAN. As in the experiment described in \cref{sec_approx_fun}, we replace the activation units of the generative and discriminator networks by a rational function with trainable coefficients (see~\cref{fig_init_rat}). We initialize the activation functions in the training phase with the best rational function that approximates the ReLU function on $[-1,1]$.

\begin{figure}[htbp]
\centering
\begin{minipage}{0.8\textwidth}
\centering
\begin{overpic}[width=0.225\textwidth,trim={0 952px 140px 0},clip]{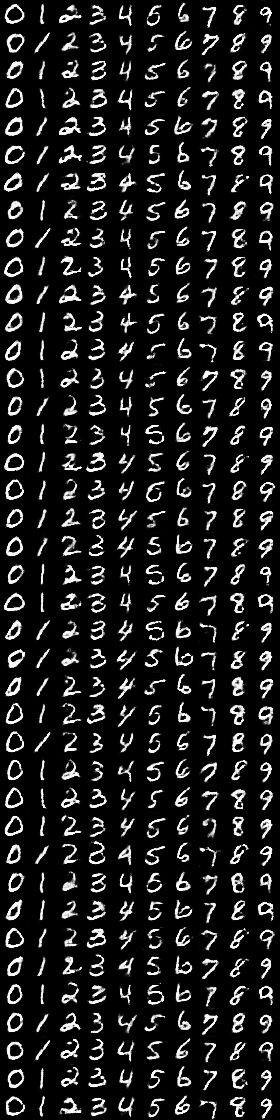}
\put(25,-11){epoch 5}
\put(-15,35){\rotatebox{90}{ReLU}}
\end{overpic}
\begin{overpic}[width=0.225\textwidth,trim={0 952px 140px 0},clip]{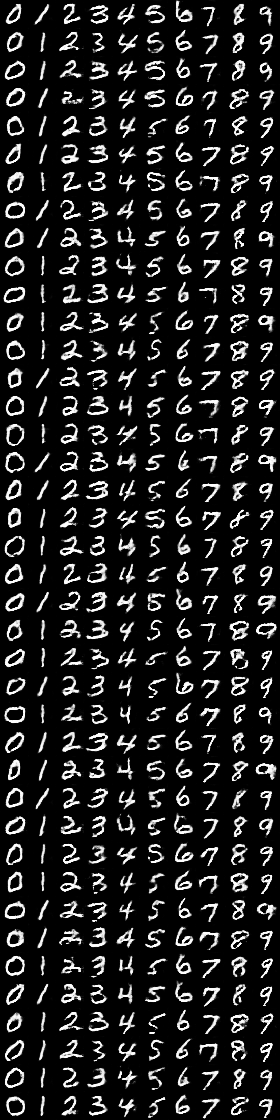}
\put(20,-11){epoch 10}
\end{overpic}
\begin{overpic}[width=0.225\textwidth,trim={0 952px 140px 0},clip]{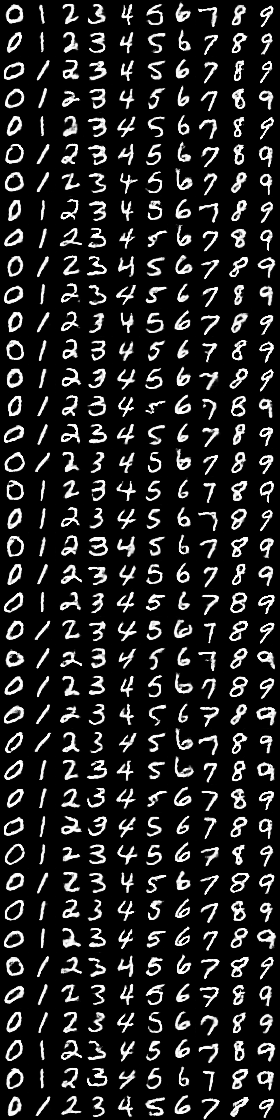}
\put(20,-11){epoch 15}
\end{overpic}
\begin{overpic}[width=0.225\textwidth,trim={0 952px 140px 0},clip]{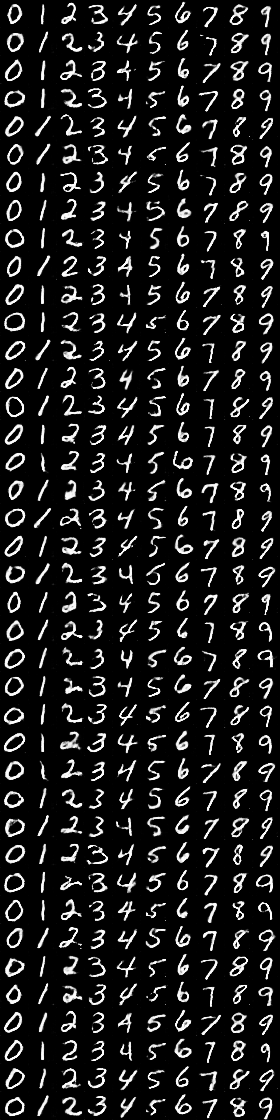}
\put(20,-11){epoch 20}
\end{overpic}\\
\vspace{13px}
\begin{overpic}[width=0.225\textwidth,trim={0 952px 140px 0},clip]{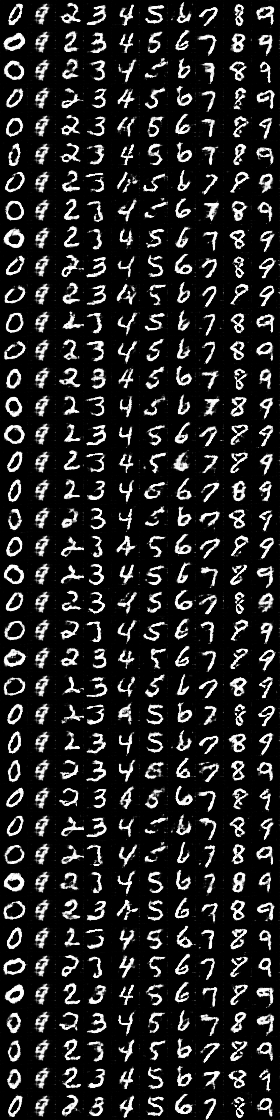}
\put(25,-11){epoch 5}
\put(-15,30){\rotatebox{90}{Rational}}
\end{overpic}
\begin{overpic}[width=0.225\textwidth,trim={0 952px 140px 0},clip]{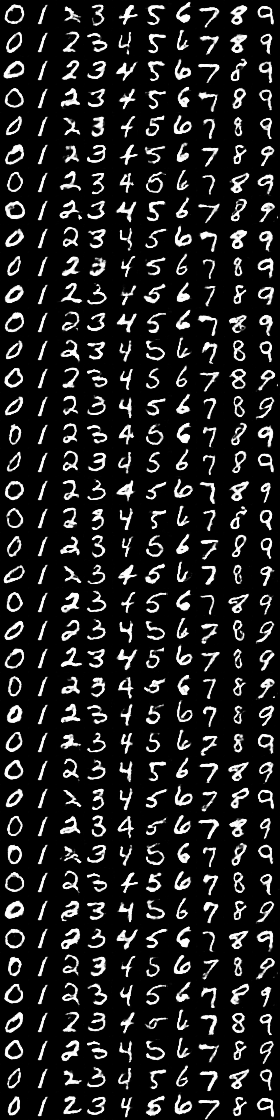}
\put(20,-11){epoch 10}
\end{overpic}
\begin{overpic}[width=0.225\textwidth,trim={0 952px 140px 0},clip]{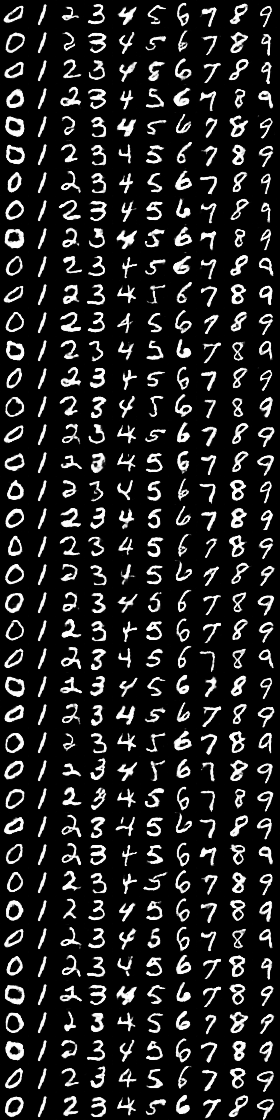}
\put(20,-11){epoch 15}
\end{overpic}
\begin{overpic}[width=0.225\textwidth,trim={0 952px 140px 0},clip]{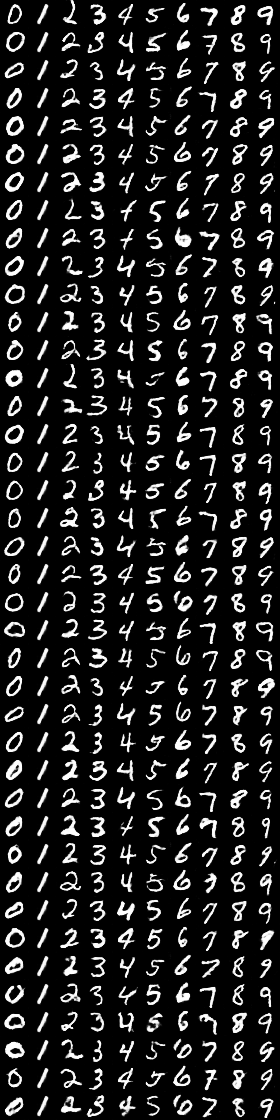}
\put(20,-11){epoch 20}
\end{overpic}
\end{minipage}%
\begin{minipage}{0.18\textwidth}
\centering
\begin{overpic}[width=\textwidth,trim={0 756px 140px 0},clip]{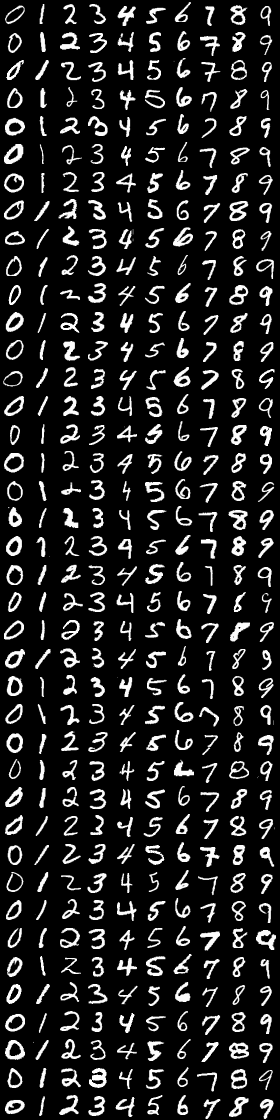}
\put(3,-5){MNIST images}
\end{overpic}
\end{minipage}
\vspace{0.5cm}
\caption{Digits generated by a ReLU (top) and rational (bottom) auxiliary classifier generative adversarial network. The right panel contains samples from the first five classes of the MNIST dataset for comparison.}
\label{fig_MNIST_gan}
\end{figure}

We show images of digits from the first five classes generated by a ReLU and rational GANs at different epochs of the training in \cref{fig_MNIST_gan} (the samples are generated randomly and are not manually selected). We observe that a rational network can generate realistic images with a broader range of features than the ReLU network, as illustrated by the presence of bold numbers at the epoch 20 in the bottom panel of \cref{fig_MNIST_gan}. However, the digits one generated by the rational network are identical, suggesting that the rational GAN suffers from mode collapse. It should be noted that generative adversarial networks are notoriously tricky to train~\cite{goodfellow2016deep}. The hyper-parameters of the reference model are intensively tuned for a piecewise linear activation function (as shown by the use of Leaky ReLU in the discriminator network). Moreover, many stabilization methods have been proposed to resolve the mode collapse and non-convergence issues in training, such as Wasserstein GAN~\cite{arjovsky2017wasserstein}, Unrolled Generative Adversarial Networks~\cite{metz2016unrolled}, and batch normalization~\cite{ioffe2015batch}. These techniques could be explored and combined with rational networks to address the mode collapse issue observed in this experiment.

\section{Conclusions}
We have investigated rational neural networks, which are neural networks with smooth trainable activation functions based on rational functions. Theoretical statements demonstrate the improved approximation power of rational networks in comparison with ReLU networks. In practice, it seems beneficial to select the activation function as very low-degree rationals, making training more computationally efficient. We emphasize that it is simple to implement rational networks in existing deep learning architectures, such as TensorFlow, together with the ability to have trainable activation functions.

There are many future research directions exploring the potential applications of rational networks in fields such as image classification, time series forecasting, and generative adversarial networks. These applications already employ nonstandard activation functions to overcome various drawbacks of ReLU. Another exciting and promising field is the numerical solution and data-driven discovery of partial differential equations with deep learning. We believe that popular techniques such as physics-informed neural networks~\cite{raissi2019physics} could benefit from rational neural networks to improve the robustness and performances of PDE solvers, both from a theoretical and practical viewpoint. 

\section*{Broader Impact}
Neural networks have applications in diverse fields such as facial recognition, credit-card fraud, speech recognition, and medical diagnosis. There is a growing understanding of the approximation power of neural networks, which is adding theoretical justification to their use in societal applications. We are particularly interested in the future applicability of rational neural networks in discovering and solving of partial differential equations (PDEs). Neural networks, in particular rational neural networks, have the potential to revolutionize fields where PDE models derived by mechanistic principles are lacking. 

\begin{ack}
The authors thank the National Institute of Informatics (Japan) for funding a research visit, during which this project was initiated. We thank Gilbert Strang for making us aware of Telgarsky's paper~\cite{telgarsky2017neural}. We also thank Matthew Colbrook and Nick Trefethen for their suggestions on the paper. This work is supported by the EPSRC Centre For Doctoral Training in Industrially Focused Mathematical Modelling (EP/L015803/1) in collaboration with Simula Research Laboratory. The work of the third author is supported by the National Science Foundation grant no. 1818757.
\end{ack}

\small

\bibliographystyle{plain}
\bibliography{biblio}

\normalsize

\appendix
\section{Supplementary Material}

\subsection{Deferred proofs of \cref{sec_approx_relu_rat}}

We first show that a rational function can approximate the absolute value function $|x|$ on $[-1,1]$ with square-root exponential convergence.

\begin{lemma}  \label{th_abs_approx}
For any integer $k\geq 0$, we have
\[
\min_{r\in\mathcal{R}_{k,k}} \max_{x\in[-1,1]} \left| |x| - xr(x)\right| \leq 4 e^{-\pi\sqrt{k/2}},
\]
where $\mathcal{R}_{k,k}$ is the space of rational functions of type at most $(k,k)$. Thus, $xr(x)$ is a rational approximant to $|x|$ of type at most $(k+1,k)$. 

Moreover, if $k=\prod_{i=1}^pk_i$ for some $p\geq 1$ and integers $k_1,\ldots,k_p\geq 2$, then $r$ can be written as $r=R_p\circ\cdots\circ R_1$, where $R_i\in\mathcal{R}_{k_i,k_i}$.
\end{lemma}

\begin{proof}
Let $0<\ell<1$ be a real number and consider the sign function on the domain $[-1,-\ell]\cup[\ell,1]$, i.e., 
\[
{\rm \sign}(x) = 
\begin{cases}
-1, & x\in[-1,-\ell], \\
+1, & x\in [\ell,1].
\end{cases}
\]
By~\cite[Equation~(33)]{beckermann2017singular}, we find that for any $k\geq 0$,
\[
\min_{r\in\mathcal{R}_{k,k}} \max_{x\in[-1,-\ell]\cup[\ell,1]} |\sign(x) - r(x)| \leq 4\left[\exp\left(\frac{\pi^2}{2\log(4/\ell)}\right)\right]^{-k}.
\]
Let $r(x)$ be the rational function of type $(k,k)$ that attains the minimum~\cite[Equation~(12)]{beckermann2017singular}. We refer to such $r(x)$ as the Zolotarev sign function. It is given by
\[
r(x) = Mx \frac{\prod_{j=1}^{\lfloor (k-1)/2\rfloor} x^2+c_{2j}}{\prod_{j=1}^{\lfloor k/2\rfloor} x^2+c_{2j-1}}, \quad c_j = \ell^2\frac{{\rm sn^2}(jK(\kappa)/k ; \kappa)}{1-{\rm sn^2}(jK(\kappa)/k ; \kappa)}.
\]
Here, $M$ is a real constant selected so that ${\rm sign}(x)-r(x)$ equioscillates on $[-1,-\ell]\cup[\ell,1]$, $\kappa = \sqrt{1-\ell^{2}}$, ${\rm sn}(\cdot)$ is the first Jacobian elliptic function, and $K$ is the complete elliptic integral of the first kind. Since $|x| = x\cdot{\rm sign}(x)$ we have the following inequality,
\[
\begin{aligned} 
\max_{x\in[-1,-\ell]\cup[\ell,1]}\left| |x| - xr(x)\right| &= \max_{x\in[-1,-\ell]\cup[\ell,1]}\left| x\cdot\sign(x) - xr(x)\right|\\
&\leq \max_{x\in[-1,-\ell]\cup[\ell,1]}\left| \sign(x) - r(x)\right|.
\end{aligned} 
\]
The last inequality follows because $|x|\leq 1$ on $[-1,-\ell]\cup[\ell,1]$. Moreover, since $xr(x)\geq 0$ for $x\in[-1,1]$ (see~\cite[Equation~(12)]{beckermann2017singular}) we have
\[
\max_{x\in[-\ell,\ell]}\left| |x| - xr(x)\right| \leq \max_{x\in[-\ell,\ell]}\left| x\right|\leq \ell. 
\]
Therefore, 
\[
\max_{x\in[-1,1]}\left| |x| - xr(x)\right|\leq \max\left\{\ell, 4\left[\exp\left(\frac{\pi^2}{2\log(4/\ell)}\right)\right]^{-k}\right\}. 
\]
Now, select $0<\ell<1$ to minimize this upper bound. One finds that $\ell = 4\exp(-\pi\sqrt{k/2})$ and the result follows immediately. 

For the final claim, let $r$ be the Zolotarev sign function $Z_{k}(\cdot\,;\ell)$ of type $(k,k)$ on $[-1,-\ell]\cup[\ell,1]$, with $k=\prod_{i=1}^p k_i$. By definition, $Z_{k}(\cdot;\ell)$ is the best rational approximation of degree $k$ to the $\sign$ function on $[-1,-\ell]\cup[\ell,1]$. We know from~\cite{lebedev1977zolotarev,nakatsukasa2016computing} that there exist $p$ Zolotarev sign functions $R_1,\ldots,R_p$, where each $R_i$ is of type $(k_i,k_i)$, such that
\begin{equation} \label{eq_composition_Zolotarev}
r(x) := Z_{k}(x;\ell) = R_p(\cdots(R_2(R_1(x))\cdots).
\end{equation}
\end{proof}

The proof of \cref{lem_relu_rat} is a direct consequence of the previous lemma and the properties of Zolotarev sign functions.

\begin{proof}[Proof of \cref{lem_relu_rat}]
Let $0<\epsilon<1$, $0<\ell<1$, $k\geq 1$, and $r$ be the Zolotarev sign function $Z_{3^k}(\cdot\,;\ell)$ of type $(3^k,3^k-1)$. Again from~\cite{lebedev1977zolotarev,nakatsukasa2016computing}, we see that there exist $k$ Zolotarev sign functions $R_1,\ldots,R_k$ of type $(3,2)$ such that their composition equals $Z_{3^k}(x;\ell)$, i.e.,
\begin{equation} \label{eq_composition_Zolotarev_2}
r(x) := Z_{3^k}(x;\ell) = R_k(\cdots (R_2(R_1(x))\cdots).
\end{equation}
Following the proof of \cref{th_abs_approx}, we have the inequality
\begin{equation} \label{eq_proof_relu_rat}
\max_{x\in[-1,1]} \left| |x| - xr(x)\right| \leq 4 e^{-\pi\sqrt{3^k/2}},
\end{equation}
where we chose $\ell=4\exp(-\pi\sqrt{3^k/2})$. Now, we take
\begin{equation} \label{eq_proof_relu_rat_k}
k=\left\lceil \frac{\ln(2/\pi^2)+2\ln(\ln(4/\epsilon))}{\ln(3)} \right\rceil,
\end{equation}
so that the right-hand side of~\cref{eq_proof_relu_rat} is bounded by $\epsilon$. Finally, we use the identity
\[
\relu(x) = \frac{|x|+x}{2},\quad x\in\mathbb{R},
\]
to define a rational approximation to the ReLU function on the interval $[-1,1]$ as
\[
\tilde{r}(x)=\frac{1}{2}\left(\frac{xr(x)}{1+\epsilon}+x\right).
\]
Therefore, we have the following inequalities for $x\in[-1,1]$,
\begin{align*}
|\relu(x)-\tilde{r}(x)| &= \frac{1}{2}\left||x|-\frac{xr(x)}{1+\epsilon}\right|\leq \frac{1}{2(1+\epsilon)}(||x|-xr(x)|+\epsilon|x|)\\
&\leq \frac{\epsilon}{1+\epsilon}\leq \epsilon.
\end{align*}
Then, $r$ is a composition of $k$ rational functions of type $(3,2)$ and can be represented using at most $7k$ coefficients (see~\cref{eq_composition_Zolotarev}). Moreover, using~\cref{eq_proof_relu_rat_k}, we see that $k=\mathcal{O}(\log(\log(1/\epsilon)))$, which means that $\tilde{r}$ is representable by a rational network of size $\mathcal{O}(\log(\log(1/\epsilon)))$. Finally, $|\tilde{r}(x)|\leq 1$ for $x\in[-1,1]$.
\end{proof}

The upper bound on the complexity of the neural network obtained in \cref{lem_relu_rat} is optimal, as proved by Vyacheslavov~\cite{vyacheslavov1975uniform}.

\begin{theorem}[Vyacheslavov] \label{th_approx_abs}
The following inequalities hold:
\begin{equation}
C_1e^{-\pi\sqrt{k}}\leq \max_{x\in[-1,1]}||x|-r_k(x)| \leq C_2 e^{-\pi\sqrt{k}}, \qquad k\geq 0,
\end{equation}
where $r_k$ is the best rational approximation to $|x|$ in $[-1,1]$ from $\mathcal{R}_{k,k}$. Here,  $C_1,C_2>0$ are constants that are independent of $k$.
\end{theorem}

We first deduce the following corollary, giving lower and upper bounds on the optimal rational approximation to the ReLU function.

\begin{corollary} \label{cor_best_relu}
The following inequalities hold: 
\begin{equation}
\frac{C_1}{2}e^{-\pi\sqrt{k}}\leq \|\relu-r_k\|_{\infty} \leq \frac{C_2}{2} e^{-\pi\sqrt{k}}, \qquad k\geq 0,
\end{equation}
where $r_k$ is the best rational approximation to ReLU on $[-1,1]$ in $\mathcal{R}_{k,k}$ and $C_1,C_2>0$ are constants given by \cref{th_approx_abs}.
\end{corollary}

\begin{proof}
Let $k$ be an integer and let $r_k\in\mathcal{R}_{k,k}$ be any rational function of degree $\leq k$. Now, define $r_\abs(x)=2r_k(x)-x$. Since $\relu(x)=(|x|+x)/2$, we have
\[
\|\relu-r_k\|_{\infty} = \! \max_{x\in[-1,1]} \left|\frac{1}{2}(r_\abs(x)+x)-\frac{1}{2}(|x|+x)\right| = \!\max_{x\in[-1,1]} \frac{1}{2}\!\left|r_\abs(x)-|x|\right|\geq \frac{1}{2}C_1e^{-\pi\sqrt{k}
},\]
where the inequality is from~\cref{th_approx_abs}.  Now, let $r_k\in\mathcal{R}_{k,k}$ be the best rational approximation to $|x|$ on $[-1,1]$. Now, define $r_\relu(x) = (r_k(x)+x)/2$. We find that
\[
\|\relu-r_\relu\|_{\infty} = \! \max_{x\in[-1,1]}  \left| \frac{1}{2}(|x|+x) - \frac{1}{2}(r_k(x)+x)\right|=\!  \max_{x\in[-1,1]}\frac{1}{2}\!\left||x|-r_k(x)\right|\leq \frac{1}{2}C_2e^{-\pi\sqrt{k}},
\]
which proves that the best approximation to ReLU satisfies the upper bound.
\end{proof}

We now show that a rational neural network must be at least $\Omega(\log(\log(1/\epsilon)))$ in size (total number of nodes) to approximate the ReLU function to within $\epsilon$. 

\begin{proposition} \label{prop_optimal_bound_rat_relu}
Let $0<\epsilon<1$. A rational neural network that approximates the ReLU function on $[-1,1]$ to within $\epsilon$ has size of at least $\Omega(\log(\log(1/\epsilon)))$. 
\end{proposition}
\begin{proof}
Let $R: [-1,1]\rightarrow \mathbb{R}$ be a rational neural network with $k_1,\ldots,k_M\geq 1$ nodes at each of its $M$ layers, and assume that its activation functions are rational functions of type at most $(r_P,r_Q)$. Let $d_r=\max(r_P,r_Q)$ be the maximum of the degrees of the activation functions of $R$. Such a network has size $\sum_{i=1}^M k_i$. Note that $R$ itself is a rational function of degree $d$, where from additions and compositions of rational functions we have $d \leq d_r^M\prod_{i=1}^M k_i$.  If $R$ is an $\epsilon$-approximation to the ReLU function on $[-1,1]$, we know by~\cref{cor_best_relu} that 
\begin{equation} \label{eq_degree_relu}
\frac{C_1}{2}e^{-\pi\sqrt{d}}\geq\epsilon, \qquad d\geq \left(\frac{1}{\pi}\ln\left(\frac{C_1}{2\epsilon}\right)\right)^2.
\end{equation}
The statement follows by minimizing the size of $R$, i.e., $\sum_{i=1}^M k_i$ subject to 
\[d_r^M\prod_{i=1}^M k_i\geq  \left(\frac{1}{\pi}\ln\left(\frac{C_1}{2\epsilon}\right)\right)^2.
\]
That is,
\begin{equation} \label{eq_KKT}
\sum_{i=1}^M\ln(k_i)+M\ln(d_r) \geq 2\ln\left(\ln\left(\frac{C_1}{2\epsilon}\right)\right)-2\ln(\pi).
\end{equation}
We introduce a Lagrange multiplier $\lambda\in\mathbb{R}$ and define the Lagrangian of this optimization problem as
\[
\mathcal{L}(k_1,\ldots,k_M,\lambda)=\sum_{i=1}^M k_i+\lambda\left[2\ln\left(\ln\left(\frac{C_1}{2\epsilon}\right)\right)-2\ln(\pi) - \sum_{i=1}^M\ln(k_i) - M\ln(d_r)\right].
\]
One finds using the Karush--Kuhn--Tucker conditions~\cite{kuhn1951} that $k_1=\cdots=k_M=\lambda$. Then, using \cref{eq_KKT}, we find that $\lambda$ satisfies
\begin{equation} \label{eq_lambda}
\ln(\lambda) \geq \frac{2}{M}\left[\ln\left(\ln\left(\frac{C_1}{2\epsilon}\right)\right)-\ln(\pi)\right] - \ln(d_r) =: \ln(\lambda^*).
\end{equation}
Therefore, the rational network $R$ with $M$ layers that approximates the ReLU function to within $\epsilon$ on $[-1,1]$ has a size of at least $s(M):=M\lambda^*$, where $\lambda^*$ is given by~\cref{eq_lambda} and depends on $M$. We now minimize $s(M)$ with respect to the number of layers $M\geq 1$. We remark that minimizing $s$ is equivalent of minimizing $\ln(s)$, where
\[
\ln(s(M)) = \ln(M) + \ln(\lambda^*) = \ln(M)+\frac{2}{M}\left[\ln\left(\ln\left(\frac{C_1}{2\epsilon}\right)\right)-\ln(\pi)\right] - \ln(d_r).
\]
One finds that one should take $k_1 =\cdots=k_M=\lambda^*= \mathcal{O}(1)$ and $M =\Omega(\log(\log(1/\epsilon)))$. The result follows. 
 \end{proof}

We now show that ReLU neural networks can approximate rational functions.

\begin{proof} [Proof of \cref{lem_relu_approx}]
Let $0<\epsilon<1$ and $R:[-1,1]\to[-1,1]$ be a rational function. Take $\tilde{R}(x) = R( 2x-1)$, which is still a rational function. Without loss of generality, we can assume that $\tilde{R}$ is an irreducible rational function (otherwise cancel factors till it is irreducible). Since $\tilde{R}$ is a rational, it can be written as $\tilde{R}=p/q$ with $\max_{x\in[0,1]}|q(x)|=1$. Moreover, we know that $\tilde{R}(x)\in[-1,1]$ for $x\in[0,1]$ so we can assume that $q(x)\geq 0$ for $x\in [0,1]$ (it is either positive or negative by continuity). Since $R$ is continuous on $[-1,1]$, there is an integer $n\geq 1$ such that $q(x)\in[2^{-n},1]$ for $x\in[0,1]$. Furthermore, we find that $|p(x)|\leq 1$ for $x\in[0,1]$ because $|R(x)|\leq 1$ and $|q(x)|\leq 1$ for $x\in[0,1]$. By~\cite[Theorem~1.1]{telgarsky2017neural}, there exists a ReLU network $f:[0,1]\to \mathbb{R}$ of size $\mathcal{O}(n^7\log(1/\epsilon)^3)$ such that
\[
\max_{x\in[0,1]}\left|f(x)-\frac{p(x)}{q(x)}\right|\leq \frac{\epsilon}{2}.
\]
We now define a scaled ReLU network $\tilde{f}(x)=f(x)/(1+\epsilon/2)$ such that $|\tilde{f}(x)|\leq 1$ for $x\in[0,1]$.
Therefore, for all $x\in[0,1]$,
\[
\left|\tilde{f}(x)-\tilde{R}(x)\right| = \left|\frac{f(x)}{1+\epsilon/2}-\frac{p(x)}{q(x)}\right|
\leq \frac{1}{1+\epsilon/2}\left(\left|f(x)-\frac{p(x)}{q(x)}\right|+\frac{\epsilon}{2}\left|\frac{p(x)}{q(x)}\right|\right)
\leq \epsilon.
\]
Therefore, $x\mapsto \tilde{f}((x+1)/2)$ is a ReLU neural network of size $\mathcal{O}(\log(1/\epsilon)^3)$ that is an $\epsilon$-approximation to $R$ on $[-1,1]$.
\end{proof}

We can now prove \cref{th_rat_network} that shows how rational neural networks can approximate ReLU networks and vice versa. The structure of the proof closely follows~\cite[Lemma~1.3]{telgarsky2017neural}. 

\begin{proof}[Proof of \cref{th_rat_network}] \leavevmode
The statement of~\cref{th_rat_network} comes in two parts, and we prove them separately. 
1.~~Consider the subnetwork $H$ of the rational network $R$, consisting of the layers of $R$ up to the $J$th layer for some $1\leq J\leq M-1$. Let $H_\relu$ denote the ReLU network obtained by replacing each rational function $r_{ij}$ in $H$ by a ReLU network approximation $f_{r_{ij}}$ at a given tolerance $\epsilon_j>0$ for $1\leq j\leq J$ and $1\leq i\leq k_j$, such that $|H_\relu(x)|\leq 1$ for $x\in[-1,1]$ (see Lemma~\ref{lem_relu_approx}). Let $x\mapsto r_{i,J+1}(a_{i,J+1}^\top H(x)+b_{i,J+1})$ be the output of the rational network $R$ at layer $J+1$ and node $i$ for $1\leq i\leq k_J$. Now, approximate node $i$ in the $(J+1)$st layer by a ReLU network $f_{r_i,J+1}$ with tolerance $\epsilon_{J+1}>0$ (see Lemma~\ref{lem_relu_approx}). The approximation error $E_{i,J+1}$ between the rational and the approximating ReLU network at layer $J+1$ and node $i$ satisfies
\begin{align*}
E_{i,J+1} &= |f_{r_{i,J+1}}(a_{i,J+1}^\top H_\relu(x)+b_{i,J+1}) - r_{i,J+1}(a_{i,J+1}^\top H(x)+b_{i,J+1})| \\
& \leq \underbrace{|f_{r_{i,J+1}}(a_{i,J+1}^\top H_\relu(x)+b_{i,J+1}) - r_{i,J+1}(a_{i,J+1}^\top H_\relu(x)+b_{i,J+1})|}_{(1)}\\
& + \underbrace{|r_{i,J+1}(a_{i,J+1}^\top H_\relu(x)+b_{i,J+1}) - r_{i,J+1}(a_{i,J+1}^\top H(x)+b_{i,J+1})|}_{(2)}. 
\end{align*}
The first term is bounded by 
\[
(1) \leq \max_{x\in[-1,1]} \left| r_{i,J+1}(x)-f_{r_{i,J+1}}\right| \leq \epsilon_{J+1},
\]
since $\left|a_{i,J+1}^\top H_\relu(x)+b_{i,J+1}\right| \leq \|a_{i,J+1}\|_1+|b_{i,J+1}|\leq 1$ by assumption. The second term is bounded as the Lipschitz constant of $r_{i,J+1}$ is at most $L$. That is,  
\[
(2) \leq L\|a_{i,J+1}\|_1\max_{x\in[-1,1]^d} \left\|H_\relu(x)-H(x)\right\|_\infty \leq  L\max_{x\in[-1,1]^d} \left\|H_\relu(x)-H(x)\right\|_\infty,
\]
where we used the fact that $\|a_{i,J+1}\|_1\leq 1$ and $\|H_\relu(x)\|_\infty \leq 1$ for $x\in[-1,1]^d$. We find that we have the following set of inequalities:  
\[
\max_{1\leq i\leq k_{j+1}} E_{i,j+1}\leq L \max_{1\leq i\leq k_{j}} E_{i,j}+\epsilon_{j+1}, \qquad 1\leq i\leq k_j, \quad 1\leq j\leq J+1,
\] 
with $E_{i,0} = 0$.  If we select $\epsilon_j=\epsilon L^{j-J-1}/(J+1)$, then we find that $\max_{1\leq i\leq k_{J+1}} E_{i,J+1} \leq \epsilon$. When $J = M-1$, the ReLU network approximates the original rational network, $R$, and the ReLU network has size
\[
\mathcal{O}\left(k \sum_{j=1}^M\log\left(\frac{M}{L^{j-M}\epsilon}\right)^3\right).
\]
where we used the fact that $k_j\leq k$ for $1\leq j\leq M$. This can be simplified a little since 
\[
\sum_{j=1}^M \log\left(\frac{M}{L^{j-M}\epsilon}\right)^3=\sum_{j=1}^M\left(\log(ML^M/\epsilon)+j\log(1/L)\right)^3=\mathcal{O}\!\left(M\log(ML^M/\epsilon)^3\right).
\]

2.~~Telgarsky proved in~\cite[Lemma~1.3]{telgarsky2017neural} that if $H_R$ is a neural network obtained by replacing all the ReLU activation functions in $f$ by rational functions $R$ for $1\leq j\leq M$, which satisfies $R(x)\in[-1,1]$ and $|R(x)-\relu(x)|\leq \epsilon/M$ for $x\in[-1,1]$, then 
\[
\max_{x\in[-1,1]^d}|f(x)-H_R(x)|\leq \epsilon.
\] 
Let $R$ be a rational neural network approximating ReLU with a tolerance of $\epsilon/M$, constructed by~\cref{lem_relu_rat}. Then, $R$ is rational network of size $\mathcal{O}(\log(\log(M/\epsilon)))$ and thus, $H_R$ is a rational neural network of size $\mathcal{O}(Mk\log(\log(M/\epsilon)))$.
\end{proof}

\subsection{Deferred proofs of \cref{sec_func_rat}}
Here, we show that the construction in~\cref{lem_relu_rat} can approximate any piecewise linear function on $[-1,1]$.

\begin{proposition} \label{th_approx_piecewise_rat}
Let $0<\epsilon<1$ and let $g:[0,1]\rightarrow\mathbb{R}$ be any continuous piecewise linear function with $m\geq 1$ breakpoints and Lipschitz constant $L>0$. Then, there exists a rational neural network $R:[0,1]\rightarrow \mathbb{R}$ of size at most 
\[
\mathcal{O}(m\log(\log(L/\epsilon)))
\]
such that $\max_{x\in[0,1]} |g(x)-R(x)|\leq \epsilon$.
\end{proposition}

\begin{proof}
Let $0\leq b_1<\cdots<b_M\leq 1$ be the breakpoints of $g$. In a similar way to the proof of~\cite[Proposition 1]{yarotsky2017error}, we first express $\rho$ as the following sum:
\begin{equation} \label{eq_proof_piecewise_rat}
g(x)=c_0\relu(b_1-x)+\sum_{j=1}^mc_j\relu(x-b_j) + c_{m+1},
\end{equation}
for some constants $c_0,\ldots,c_{m+1}\in\mathbb{R}$. Therefore, $g$ can be exactly represented using a ReLU network with $m+1$ nodes and one layer, i.e., 
\[
g(x) = \begin{pmatrix}
c_0 & c_1& \cdots & c_m
\end{pmatrix}
\begin{pmatrix}
\relu(-x+b_1)\\
\relu(x-b_1)\\
\vdots\\
\relu(x-b_m)
\end{pmatrix} + c_{m+1}.
\]
Since $g$ has a Lipschitz constant of $L$, we find that $|c_0|\leq L$ and $\sum_{j=1}^m |c_j| \leq L$. Using~\cref{lem_relu_rat} we can approximate a ReLU function on $[-1,1]$ with tolerance $\epsilon/(2L)$ by a rational network $R_{\relu}$ of size $\mathcal{O}(\log(\log(2L/\epsilon)))$. Now, we construct $R:[0,1]\rightarrow \mathbb{R}$ as a rational network obtained by replacing the ReLU functions in $g$ by $R_{\relu}$. We have the following error estimate: 
\[
\max_{x\in[0,1]} |g(x)-R(x)| \leq |c_0|\|\relu-R_{\relu}\|_{\infty}+\sum_{j=1}^m |c_j| \|\relu-R_{\relu}\|_{\infty} \leq \frac{\epsilon}{2}+\frac{\epsilon}{2}\leq \epsilon.
\]
The result follows as $R$ is of size $\mathcal{O}(m\log(\log(L/\epsilon)))$.
\end{proof}

We remark that the size of the rational network required to approximate a piecewise linear function depends on $\epsilon$. In contrast, ReLU neural networks can represent piecewise linear functions exactly. In the next proposition, we show that a rational neural network can represent $x^n$, for some integer $n$, exactly.

\begin{proposition} \label{prop_rat_x_n}
Let $n\geq 1$, $r_P\geq 2$, and $r_Q\geq 0$.  There exists a rational network $R$, with rational activation functions of type $(r_P,r_Q)$, of size at most $5\lfloor\log_{r_P}(n)\rfloor^2+1$ such that $R(x) = x^n$ for all $x\in\mathbb{R}$.
\end{proposition}
\begin{proof}
We start by expressing $n$ in base $r_P$, i.e., 
\[
n=\sum_{\ell=0}^{\lfloor\log_{r_P}\!(n)\rfloor}c_\ell r_P^\ell, \qquad c_\ell\in\{0,1,\ldots,r_P-1\}.
\] 
This means we can represent $x^n$ as
\begin{equation} \label{eq_epr_x_n}
x^n=\prod_{\ell=0}^{\lfloor\log_{r_P}\!(n)\rfloor} x^{c_\ell r_P^\ell}.
\end{equation}
Note that $x^{c_\ell r_P^\ell}$ is just $x^{r_P}$ composed $\ell$ times as well as composed with $x^{c_\ell}$ so can be represented by a rational neural network with $\ell+1$ layers, each with one node. Therefore, all the $x^{c_\ell r_P^\ell}$ terms can be represented in rational networks that in total have size
\[
\sum_{\ell=0}^{\lfloor\log_{r_P}\!(n)\rfloor} \!\!(\ell+1) = \frac{1}{2}(\lfloor\log_{r_P}\!(n)\rfloor)^2 + \frac{3}{2}\lfloor\log_{r_P}\!(n)\rfloor + 1.
\]
The function $x^n$ can be formed by multiplying all the $x^{c_\ell r_P^\ell}$ terms together. Since $xy = (x^2 + y^2 - (x-y)^2)/2$, there is a rational network with one layer and three nodes that represents the multiplication operation. Therefore, multiplying all the terms together requires a rational network of size at most $3\lfloor\log_{r_P}\!(n)\rfloor$ (see \cref{eq_epr_x_n}). The result follows by noting that $x^2/2+9x/2+1\leq 5x^2+1$ for $x\geq 1$. 
\end{proof}

\begin{figure}[htbp]
\centering
\begin{overpic}[width=0.4\textwidth]{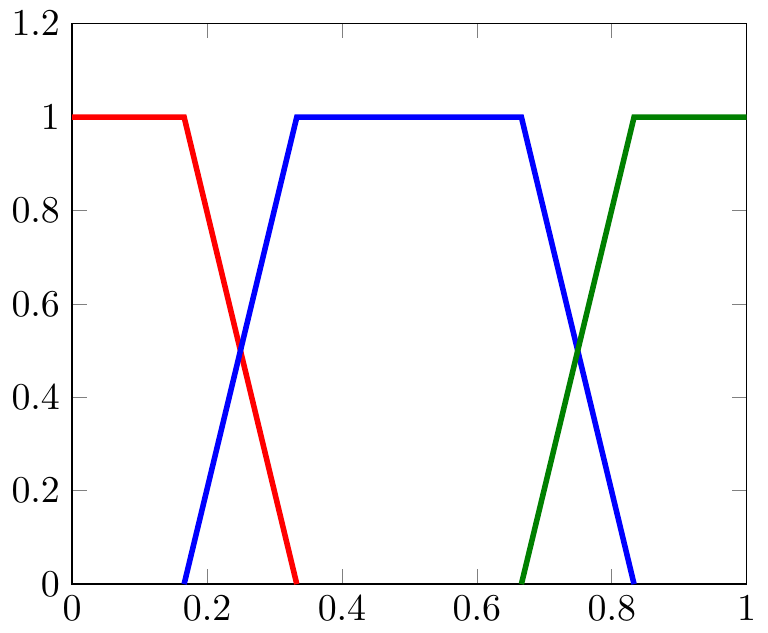}
\put(52,-5){x}
\vspace{0.2cm}
\end{overpic}
\caption{Partition of unity: $\psi_0$ (red), $\psi_1$ (blue), and $\psi_2$ (green), for $N=2$.}
\label{fig_partition}
\end{figure}

We can now prove \cref{th_approx_smooth_rat} using the two previous propositions.

\begin{proof}[Proof of \cref{th_approx_smooth_rat}]
The proof is based on the proof of~\cite[Theorem~1]{yarotsky2017error} and consists of replacing the piecewise linear functions and monomials arising in the local Taylor approximation of the function $f$ by rational networks using the previous approximation results.

Let $N\geq 1$ be an integer and consider a partition of unity of $(N+1)^d$ functions $\phi_\m$ on the domain $[0,1]^d$, i.e.,
\[\sum_{\m\in\{0,\ldots,N\}^d} \phi_\m(\x)=1, \qquad \phi_\m(\x)=\prod_{k=1}^d\psi_{m_k}(x_k), \qquad \x=(x_1,\ldots,x_d),\]
where $\m=(m_1,\ldots,m_d)$, and $\psi_{m_k}$ is given by
\[\psi_{m_k}(x)=
\begin{cases}
1, \quad & \text{if }\left|x_k-\frac{m_k}{N}\right|<\frac{1}{3N},\\
0, \quad & \text{if }\left|x_k-\frac{m_k}{N}\right|>\frac{2}{3N},\\
2-3N\left|x_k-\frac{m_k}{N}\right|, \quad & \text{otherwise}.
\end{cases}
\]
Examples of the functions $\psi_{m_k}$ are shown in~\cref{fig_partition} when $N=2$. We now define a local Taylor approximation of $f$ by 
\[
f_N(\x) =\sum_{\m\in\{0,\ldots,N\}^d}\phi_\m(\x) P_\m(\x),
\]
where $P_\m$ denotes the degree $n-1$ Taylor polynomial of $f$ at $\x=\m/N$. That is, 
\begin{equation} \label{eq_proof_th_4_expansion}
P_\m(\x)=\sum_{|\n|<n} \frac{D^\n f(\tfrac{\m}{N})}{\n !} \left(\x-\frac{\m}{N}\right)^\n,
\end{equation}
where $|\n|=\sum_{k=1}^d n_k$, $\n !=\prod_{k=1}^d n_k!$, and $(\x-\m/N)^\n=\prod_{k=1}^d(x_k-m_k/N)^{n_k}$.  Let $\x\in[0,1]^d$ and note that
\[
\text{support}(\phi_\m)\subset \left\{\x = (x_1,\ldots,x_d):\left|x_k-\frac{m_k}{N}\right|<\frac{1}{N}\right\}, \qquad \m\in\{0,\ldots,N\}^d.
\]
Hence, the approximation error between $f$ and its local Taylor approximation satisfies
\begin{align*}
|f(\x)-f_N(\x)| &= \left|\sum_{\m\in\{0,\ldots,N\}^d} \phi_\m (f(\x)-P_{\m}(\x))\right|\\
&\leq \sum_{\m:\left|x_k-\frac{m_k}{N}\right|<\frac{1}{N}}|f(\x)-P_\m(\x)|\\
&\leq \frac{2^d d^n}{n!}\left(\frac{1}{N}\right)^n \max_{|\n|=n}\esssup_{\x\in[0,1]^d}|D^\n f(\x)|\\
&\leq \frac{2^d d^n}{n!}\left(\frac{1}{N}\right)^n.
\end{align*}
We now select (see~\cite[Theorem~1]{yarotsky2017error} for a similar idea)
\[
N=\left\lceil\left(\frac{n!}{2^dd^n}\frac{\epsilon}{2}\right)^{-1/n}\right\rceil,
\]
so that 
\begin{equation} \label{eq_taylor_approx}
\max_{\x\in[0,1]^d} \left|f(\x)-f_N(\x)\right| \leq \epsilon/2.
\end{equation}
We now approximate the function $f_n$ by a rational network using~\cref{th_approx_piecewise_rat,prop_rat_x_n}. First, we write $f_N$ as
\begin{equation} \label{eq_proof_th_4_expansion_2}
f_N(\x) = \sum_{\m\in\{0,\ldots,N\}^d}\sum_{|\n|<n}a_{\m,\n}\phi_\m(\x)\left(\x-\frac{\m}{N}\right)^\n,
\end{equation}
where $|a_{\m,\n}|\leq 1$ and the monomials are uniformly bounded by $1$ (see \cref{eq_proof_th_4_expansion}). \cref{eq_proof_th_4_expansion_2} consists of at most $d^n(N+1)^d$ terms of the form $\phi_\m(\x)(\x-\m/N)^\n$. The monomial part $(\x-\m/N)^\n$ is representable by a rational network of size $\mathcal{O}(d\log(n)^2)$ using~\cref{prop_rat_x_n}, including the fact that the multiplication is a rational network with one layer and three nodes. 
Let $0<\delta<1$ be a small number, for each $m_k\in\{0,\ldots,N\}$ the piecewise linear function $\psi_{m_k}$ has a Lipschitz constant of $L=3N$. Therefore, it can be approximated with a tolerance $\delta$ by a rational network $\tilde{\psi}_{m_k}$ of size $\mathcal{O}(\log(\log(N/\delta)))$ (see \cref{th_approx_piecewise_rat}). We can assume $\|\tilde{\psi}_{m_k}\|_\infty=1$ by increasing the size of the network by a constant. This yields the following approximation error between a term in~\cref{eq_proof_th_4_expansion_2} and the rational network constructed using $\tilde{\psi}_{m_k}$:
\begin{align*}
&\left|\phi_\m(\x)\left(\x-\frac{\m}{N}\right)^\n-\prod_{k=1}^d\tilde{\psi}_{m_k}(x_k)\left(\x-\frac{\m}{N}\right)^\n\right|
\leq \left|\prod_{k=1}^d \psi_{m_k}(x_k)-\prod_{k=1}^d\tilde{\psi}_{m_k}(x_k)\right|\\
&\leq \left|\psi_{m_1}(x_1)-\tilde{\psi}_{m_1}(x_1)\right| \left|\prod_{k=2}^d \psi_{m_k}(x_k)\right|
+\left|\tilde{\psi}_{m_1}(x_1)\right| \left|\prod_{k=2}^d \psi_{m_k}(x_k)-\prod_{k=2}^d\tilde{\psi}_{m_k}(x_k)\right|\\
&\leq \left|\psi_{m_1}(x_1)-\tilde{\psi}_{m_1}(x_1)\right| 
+\left|\prod_{k=2}^d \psi_{m_k}(x_k)-\prod_{k=2}^d\tilde{\psi}_{m_k}(x_k)\right|\\
&\leq \delta +\left|\prod_{k=2}^d \psi_{m_k}(x_k)-\prod_{k=2}^d\tilde{\psi}_{m_k}(x_k)\right| \leq d\delta.
\end{align*}
Here, the final inequality is derived by repeating the previous inequalities for $x_2,\ldots,x_d$. If we denote by $\tilde{f}_N$ the rational network approximation to $f_N$ constructed above, then, for all $\x\in[0,1]^d$, we have
\begin{align*}
|f_N(\x)-\tilde{f}_N(\x)|&\leq \sum_{\m\in\{0,\ldots,N\}^d}\sum_{|\n|<n}|a_{\m,\n}|\left|\phi_\m(\x)\left(\x-\frac{\m}{N}\right)^\n-\prod_{k=1}^d\tilde{\psi}_{m_k}(x_k)\left(\x-\frac{\m}{N}\right)^\n\right|\\
&\leq 2^d d^{n+1}\delta.
\end{align*}
Therefore, we select $\delta = \epsilon/(2^{d+1} d^{n+1})$ so that $\max_{\x\in [0,1]^d} |f_N(\x)-\tilde{f}_N(\x)|\leq \epsilon/2$. Then, by \cref{eq_taylor_approx}, we have
\[
\max_{\x\in[0,1]^d} \left|f(\x)-\tilde{f}_N(\x)\right| \leq \frac{\epsilon}{2}+\frac{\epsilon}{2}\leq \epsilon.
\]
The statement of the theorem follows as the rational network $\tilde{f}_N$ has size at most
\[
\mathcal{O}(d^n(N+1)^d\log(\log(N/\delta)))=\mathcal{O}(\epsilon^{-d/n}\log(\log(1/\epsilon^{1+1/n})))=\mathcal{O}(\epsilon^{-d/n}\log(\log(1/\epsilon))).
\]

\end{proof}

\subsection{Details of the approximation experiment}
We use the TensorFlow implementation\footnote{We adapt the code that is publicly available~\cite{raissiGit}.} of the deep hidden physics model framework to build and train the identifier network $\mathcal{N}$ that approximates a solution $u$ to the KdV equation.  The true solution is computed on the domain $(x,t)\in[-20,20]\times[0,40]$ by Raissi~\cite{raissi2018deep} using the Chebfun package~\cite{driscoll2014chebfun} with a spectral Fourier discretization of $512$ and a time-step of $\Delta t = 10^{-4}$. Moreover, the solution is stored after every $2000$ time steps, giving a testing data set of approximatively $10^5$ spatio-temporal points in $[-20,20]\times[0,40]$. We then constituted the training and validation sets (of $10^4$ points each) by randomly subsampling the solution at $2\times 10^4$ points in $[-20,20]\times[0,40]$.

\begin{table}[htbp]
  \caption{Initialization coefficients of the rational activation functions.}
  \label{table_coeffs}
  \centering
  \begin{tabular}{ccccccc}
    \toprule
    $a_0$ & $a_1$ & $a_2$ & $a_3$ & $b_0$ & $b_1$ & $b_2$ \\
    \midrule
    $1.1915$ & $1.5957$ & $0.5000$ & $0.0218$ & $2.3830$ & $0.0000$ & $1.0000$ \\
    \bottomrule
  \end{tabular}
\end{table}

In a similar manner to~\cite{raissi2018deep}, we use a fully connected identification network to approximate $u$ with $4$ hidden layers with $50$ nodes per layer. The network is trained using the L-BFGS optimization algorithm with $10,\!000$ iterations. We compare three types of activation functions: ReLU, sinusoid, trainable rational functions of type $(3,2)$, and trainable polynomials of degree $3$. Furthermore, the rational activation functions are initialized to be the best approximation to the ReLU function (see~\cref{sec_experiments}), giving the initial coefficients reported in~\cref{table_coeffs}.

We represent the approximation errors between the different identification networks and the solution to the KdV equation in \cref{fig_2d_approx}. 

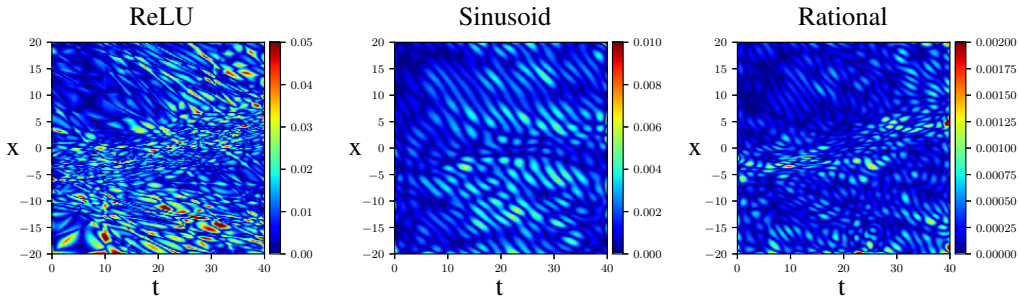
\begin{figure}[htbp]
\centering
\vspace{0.3cm}
\begin{overpic}[width=0.95\textwidth]{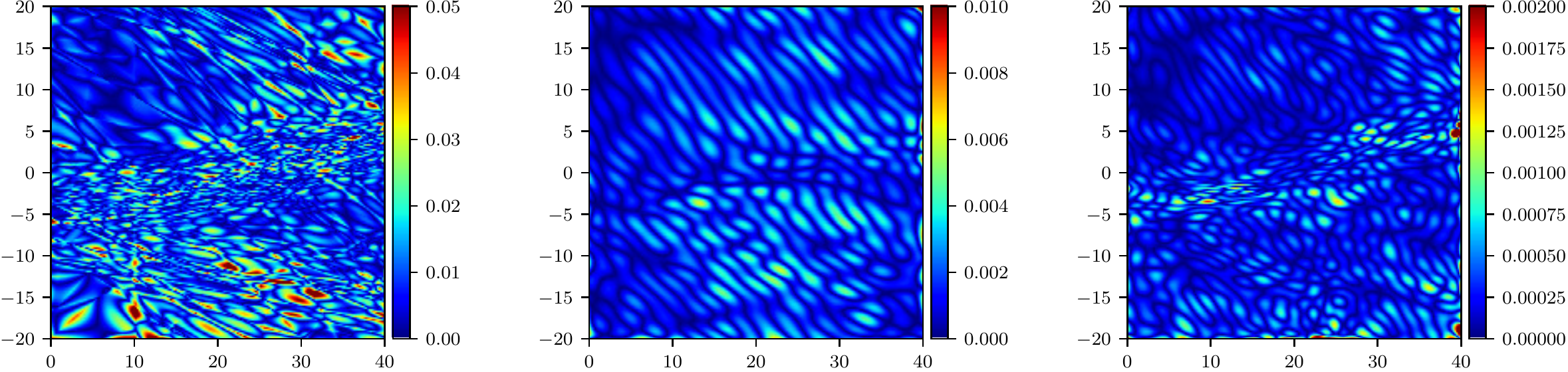}
\put(11,25){ReLU}
\put(13.3,-2){t}
\put(-1.3,12){x}
\put(44,25){Sinusoid}
\put(47.7,-2){t}
\put(33,12){x}
\put(78,25){Rational}
\put(82,-2){t}
\put(67,12){x}
\end{overpic}
\vspace{0.1cm}
\caption{Approximation errors of the neural networks with ReLU, sinusoid, and rational activation layers.}
\label{fig_2d_approx}
\end{figure}

Finally, in \cref{fig_rational_loss_2d}, we compare rational neural networks with different degree activation functions (each initialized to approximate the ReLU function using the MATLAB code 
\texttt{initial\_rational\_coeffs.m} available at~\cite{boulleGit}) and find that they all performed better than ReLU networks. While a type $(3,2)$ rational offers a good trade-off between the number of parameters and quality of approximation according to the theoretical results presented in \cref{sec_th_result}, the type of rational function might well depend on the application considered.

\begin{figure}[htbp]
\centering
\begin{overpic}[width=0.44\textwidth, trim=30 15 45 25,clip]{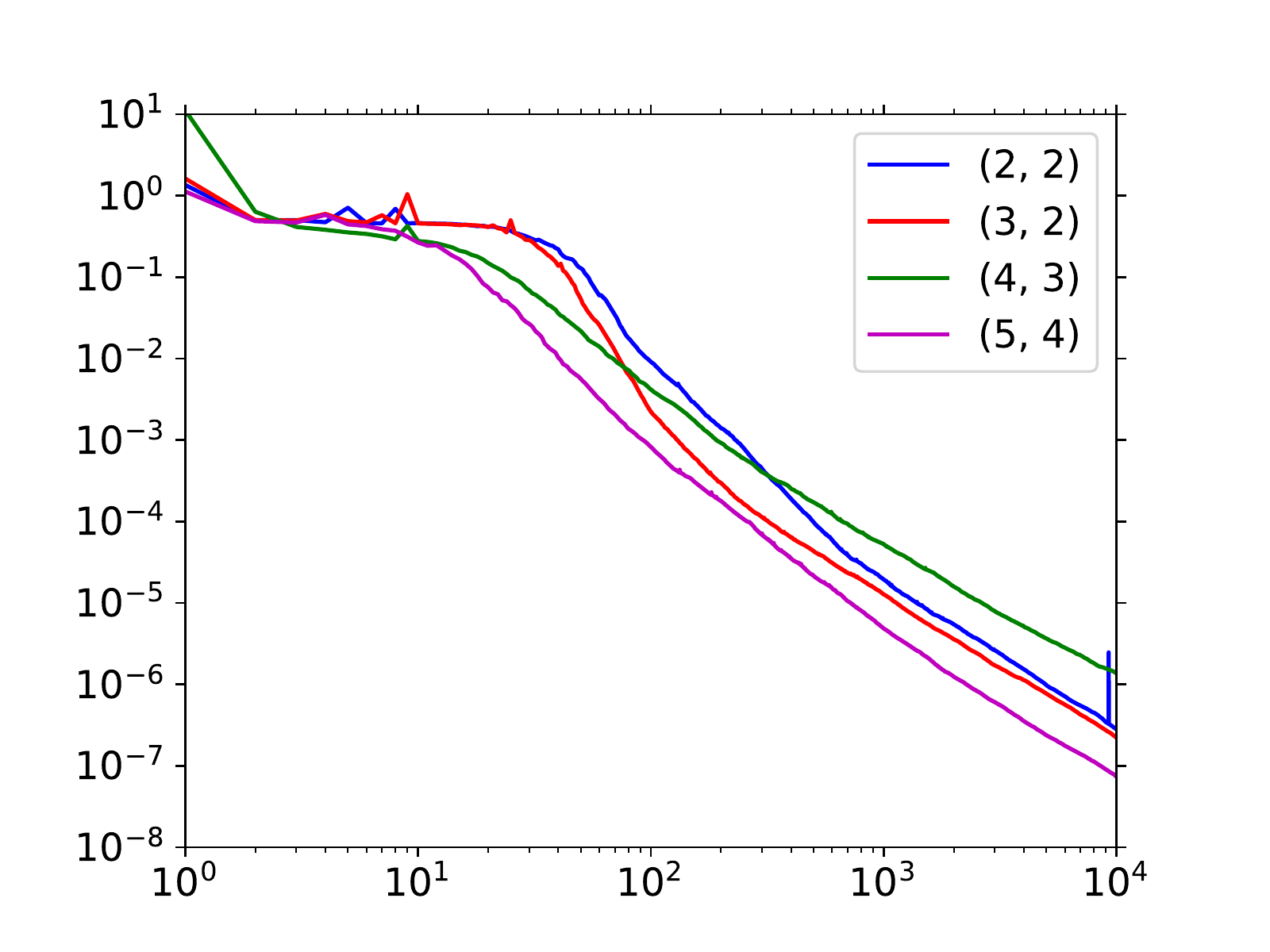}
\put(45,-5){Epochs}
\put(-9,23){\rotatebox{90}{Validation loss}}
\end{overpic}
\vspace{0.2cm}
\caption{Validation loss of rational networks of types $(2,2)$, $(3,2)$, $(4,3)$, and $(5,4)$ with respect to the number of epochs.}
\label{fig_rational_loss_2d}
\end{figure}

\subsection{Details of the GAN experiment}

We adapt the Keras example in~\cite{KerasDoc} to train an Auxiliary Classifier GAN with rational activation functions on the MNIST. The hyper-parameters used for the GAN experiment are given in~\cref{table_GAN}. Moreover, the GAN is trained on $20$ epochs with a batch size of $100$ by Adam's optimization algorithm~\cite{kingma2014adam} and the following parameters: $\alpha=0.0002$ and $\beta_1=0.5$, as suggested by~\cite{radford2015unsupervised}.

\begin{table}[htbp]
  \caption{Hyper-parameters of the GAN experiment, BN denotes the presence of a Batch normalization layer. The Generator and Discriminator networks are trained with ReLU and rational activation functions, initialized with the coefficients reported in \cref{table_coeffs}.}
  \label{table_GAN}
  \centering
  \begin{tabular}{lcccccc}
    \toprule
     \multicolumn{1}{c}{Operation} & Kernel & Strides & Features & BN & Dropout & Activation \\
    \midrule
    \multicolumn{1}{c}{Generator} \\    
    Linear & N/A & N/A & 3456 & \xmark & 0.0 & ReLU / Rational\\
    Transposed Convolution & $5\times 5$ & $1\times 1$ & 192 & \cmark & 0.0 & ReLU / Rational\\
    Transposed Convolution & $5\times 5$ & $2\times 2$ & 96 & \cmark & 0.0 & ReLU / Rational\\
    Transposed Convolution & $5\times 5$ & $2\times 2$ & 1 & \xmark & 0.0 & Tanh\\
    \multicolumn{1}{c}{Discriminator} \\    
    Convolution & $3\times 3$ & $2\times 2$ & 32 & \xmark & 0.3 & Leaky ReLU / Rational\\
    Convolution & $3\times 3$ & $1\times 1$ & 64 & \xmark & 0.3 & Leaky ReLU / Rational\\
    Convolution & $3\times 3$ & $2\times 2$ & 128 & \xmark & 0.3 & Leaky ReLU / Rational\\
    Convolution & $3\times 3$ & $1\times 1$ & 256 & \xmark & 0.3 & Leaky ReLU / Rational\\
    Linear & N/A & N/A & 11 & \xmark & 0.0 & Soft-Sigmoid\\
    \bottomrule
  \end{tabular}
\end{table}

We report in~\cref{fig_MNIST_gan_supp} samples of the $10$ classes present in the MNIST dataset (right) and images generated at the $20$th epoch by the GAN with ReLU/Leaky ReLU units (left) and rational activation functions (middle).

\begin{figure}[htbp]
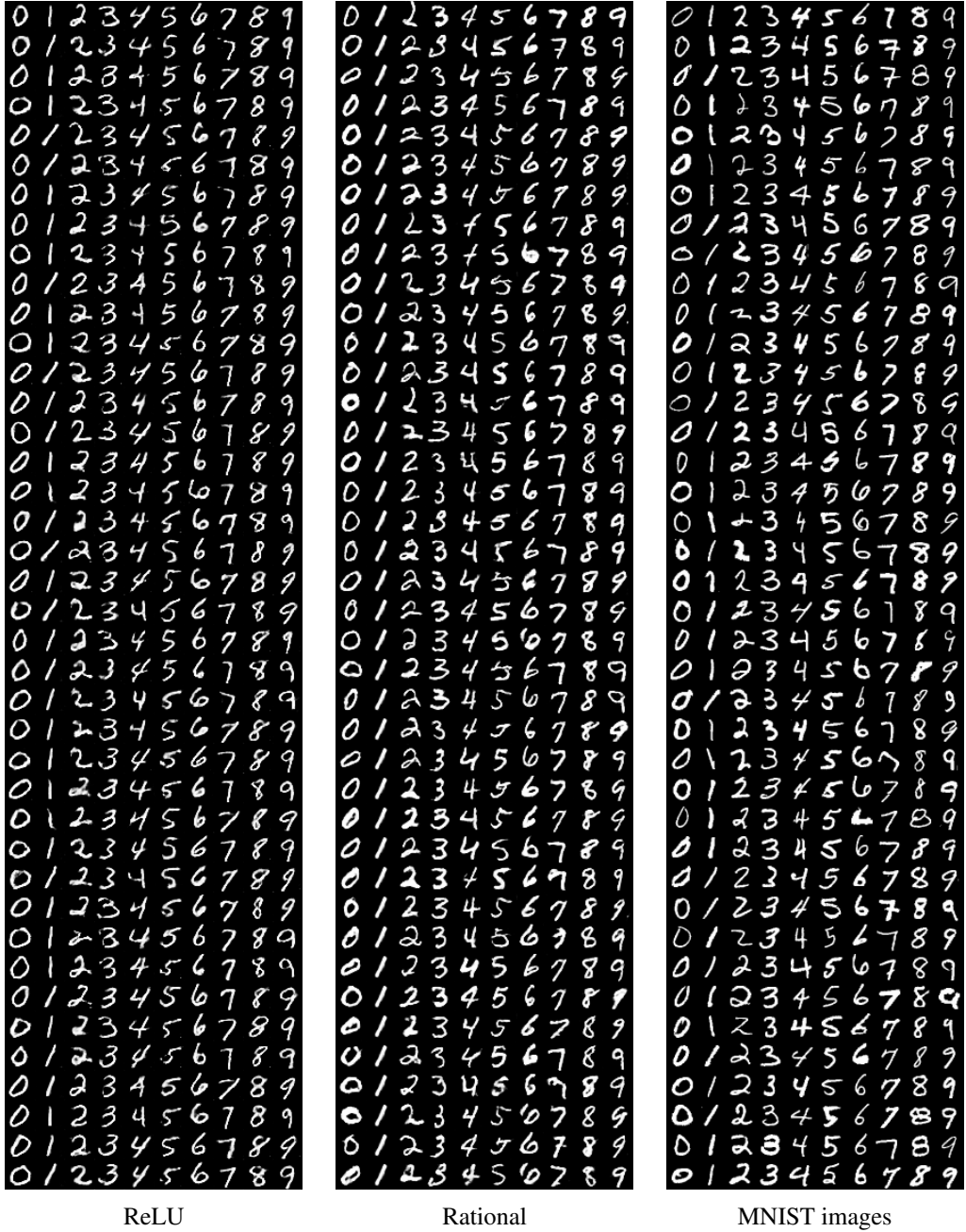

\centering
\begin{overpic}[width=0.3\textwidth,trim={0 0 0 0},clip]{Figure/gan/relu/plot_epoch_020_generated.png}
\put(10,-3){ReLU}
\end{overpic}
\hspace{0.3cm}
\begin{overpic}[width=0.3\textwidth,trim={0 0 0 0},clip]{Figure/gan/rat/plot_epoch_020_generated.png}
\put(9,-3){Rational}
\end{overpic}
\hspace{0.3cm}
\begin{overpic}[width=0.3\textwidth,trim={0 0 0 0},clip]{Figure/gan/mnist_images.png}
\put(6,-3){MNIST images}
\end{overpic}
\vspace{0.5cm}
\caption{Forty images generated by a ReLU network and a rational network after 20 epochs, together with real images from the MNIST dataset.}
\label{fig_MNIST_gan_supp}
\end{figure}

\end{document}